\documentclass[11pt, twoside, letter]{article}
\usepackage[margin = 1in]{geometry}
\usepackage[utf8]{inputenc}
\usepackage{amsmath, amssymb, amsfonts}
\usepackage{ifthen}
\usepackage[boxed, linesnumbered]{algorithm2e}
\usepackage{comment}
\usepackage[colorlinks=true,citecolor=blue,linkcolor=blue]{hyperref}
\usepackage{hyperref}
\usepackage{pgf}
\usepackage{tikz}
\usetikzlibrary{arrows,automata}

\usepackage[T1]{fontenc}
\usepackage{lmodern}
\usepackage[utf8]{inputenc}
\usepackage{authblk}
\usepackage{amstext,amssymb}
\usepackage{amsmath}
\usepackage{amsthm}
\usepackage{csquotes}
\usepackage[parfill]{parskip}
\usepackage{comment}
\usepackage{xcolor}
\usepackage{bbm}
\usepackage{mathtools}
\usepackage{ifthen}
\usepackage{graphicx,color}
\usepackage{fancybox}
\usepackage{lipsum}
\usepackage{array,float}
\usepackage{url}

\begingroup
    \makeatletter
    \@for\theoremstyle:=definition,remark,plain\do{%
        \expandafter\g@addto@macro\csname th@\theoremstyle\endcsname{%
            \addtolength\thm@preskip\parskip
            }%
        }
\endgroup
\usepackage{mathrsfs}
\usepackage{dsfont}
\definecolor{cornellred}{rgb}{0.7, 0.11, 0.11}
\definecolor{dgreen}{rgb}{0.0, 0.5, 0.0}
\definecolor{ballblue}{rgb}{0.13, 0.67, 0.8}
\definecolor{royalblue(web)}{rgb}{0.25, 0.41, 0.88}
\definecolor{bleudefrance}{rgb}{0.19, 0.55, 0.91}
\definecolor{royalazure}{rgb}{0.0, 0.22, 0.66}
\usepackage{hyperref}
\hypersetup{
    colorlinks = true,
    linkcolor=cornellred,
    citecolor=royalazure,
    linkbordercolor = {white}
}
\usepackage{cleveref}
\usepackage{enumitem}

\newtheorem{theorem}{Theorem}[section]
\newtheorem{lemma}[theorem]{Lemma}
\newtheorem*{lemma*}{Lemma}

\newtheorem{corollary}[theorem]{Corollary}

\newtheorem{claim}{Claim}
\theoremstyle{definition}
\newtheorem{definition}{Definition}[section]
\theoremstyle{definition}

\newenvironment{remark}[1][Remark]{\begin{trivlist}
\item[\hskip \labelsep {\bfseries #1}]}{\end{trivlist}}

\usepackage{accents}
\usepackage[authoryear,round]{natbib}

\DeclarePairedDelimiter{\norm}{\lVert}{\rVert}

\DeclarePairedDelimiterX{\set}[1]\{\}{#1}
\let\Pr\relax
\DeclarePairedDelimiterXPP{\Pr}[1]{\mathbb{P}}[]{}{#1}
\DeclarePairedDelimiterXPP{\Ex}[1]{\mathbb{E}}[]{}{#1}

\DeclareFontFamily{U}{matha}{\hyphenchar\font45}
\DeclareFontShape{U}{matha}{m}{n}{
      <5> <6> <7> <8> <9> <10> gen * matha
      <10.95> matha10 <12> <14.4> <17.28> <20.74> <24.88> matha12
      }{}
\DeclareSymbolFont{matha}{U}{matha}{m}{n}
\DeclareMathSymbol{\wedge}         {2}{matha}{"5E}
\DeclareMathSymbol{\vee}           {2}{matha}{"5F}

\usepackage{microtype}
\usepackage{graphicx}
\usepackage{subfigure}
\usepackage{booktabs} 
\usepackage{amsthm}
\usepackage{amsmath}
\usepackage{mathtools}

\theoremstyle{remark}

\title{Better Depth-Width Trade-offs for Neural Networks\\ through the lens of Dynamical Systems}

\author[$\dagger$]{Vaggos Chatziafratis}
\author[$\star$]{Sai Ganesh Nagarajan}
\author[$\star$]{Ioannis Panageas}
\affil[$\dagger$]{Stanford University, Department of Computer Science}
\affil[$\star$]{Singapore University of Technology and Design}

\date{}
\begin{document}

\maketitle

\begin{abstract}
The expressivity of neural networks as a function of their depth, width and type of activation units has been an important question in deep learning theory. Recently, depth separation results for ReLU networks were obtained via a new connection with dynamical systems, using a generalized notion of fixed points of a continuous map $f$, called periodic points. In this work, we strengthen the connection with dynamical systems and we improve the existing width lower bounds along several aspects. Our first main result is period-specific width lower bounds that hold under the stronger notion of $L^1$-approximation error, instead of the weaker classification error. Our second contribution is that we provide sharper width lower bounds, still yielding meaningful exponential depth-width separations, in regimes where previous results wouldn't apply. A byproduct of our results is that there exists a universal constant characterizing the depth-width trade-offs, as long as $f$ has odd periods. Technically, our results follow by unveiling a tighter connection between the following three quantities of a given function: its period, its Lipschitz constant and the growth rate of the number of oscillations arising under compositions of the function $f$ with itself.
\end{abstract}

\section{Introduction}
\label{submission}
Deep Neural Networks (NNs) with many hidden layers are now at the core of modern machine learning applications and can achieve remarkable performance that was previously unattainable using shallow networks. But why are deeper networks better than shallow? Perhaps intuitively, one can understand that the nature of computation done by deep and shallow networks is different; simple one hidden layer NNs extract independent features of the input and return their weighted sum, while deeper NNs can compute features of features, making the features computed by deeper layers no longer independent. Another line of intuition~(\cite{poole2016NIPS}), is that highly complicated manifolds in input space can actually turn into flattened manifolds in hidden space, thus helping with downstream tasks (e.g., classification).

To make the above intuitions formal and understand the benefits of depth, researchers try to understand the \textit{expressivity} of NNs and prove \textit{depth separation} results. Early results in this area  sometimes referred to as universality theorems~\citep{cybenko1989approximation, hornik1989multilayer}, state that NNs of just one hidden layer, equipped with standard activation units (e.g., sigmoids, ReLUs etc.) are ``dense'' in the space of continuous functions, meaning that any continuous function can be represented by an appropriate combination of these activation units. There is a computational caveat however, since the width of this one hidden layer network can be unbounded and grow arbitrarily with the input function. In practice, resources are bounded, hence the more meaningful questions have to do with depth separations.

This is a foundational question not only in deep learning theory but also in other computational models (e.g., boolean circuit complexity~\citep{hastad1986almost,kane2016super}) with a rich history of prior work, bringing together ideas and techniques from boolean functions, Fourier and harmonic analysis, special functions, fractal geometry, differential geometry and more recently dynamical systems and chaos. At a high level, all these works define an appropriate notion of ``complexity'' and later demonstrate how deeper models are significantly more powerful than
shallower models. A partial list of the different notions of complexity that have been considered include global curvature~\citep{poole2016NIPS} and trajectory length~\citep{raghu2017ICML}, number of activation patterns~\citep{hanin2019deep} and linear regions~\citep{montufar2014number,arora2016understanding}, fractals~\citep{malach2019deeper}, the dimension of algebraic varieties~\citep{kileel2019expressive}, Fourier spectrum of radial functions~\citep{eldan2016COLT}, number of oscillations~\citep{schmitt2000lower,Telgarsky16,telgarsky15} and periods of continuous maps~\citep{ICLR}.

In this work, we build upon the works by~\cite{Telgarsky16} that relied on the number of oscillations of continuous functions and by~\cite{ICLR} that relied on periodic orbits present in a continuous function and connections to dynamical systems to derive depth separations (see \hyperref[sec:prelim]{Section}~\ref{sec:prelim} for definitions). We pose the following question:
\begin{center}
\textit{Can we exploit further connections to dynamical systems to derive improved depth-width trade-offs?}
\end{center}

We are indeed able to do this and improve the known depth separations along several aspects:
\begin{itemize}
    \item We show that there exist real-valued functions $f$, expressible by deep NNs, for which shallower networks, even with exponentially larger width, incur large $L^1$ error instead of the weaker\footnote{The word ``weaker'' here is justified because the goal is to prove a \textit{lower} bound on the approximation error. Notice that there exist cases where the classification error is large, but the $L^1$ error is small (e.g., see example in \hyperref[fig:errors]{Figure}~\ref{fig:errors}).} notion of \textit{classification} error that was previously shown~\citep{ICLR,telgarsky15}.
    \item We obtain width lower bounds that are sharper across all regimes for the periodic orbits in $f$ and surprisingly we show that there is a universal constant characterizing the depth-width trade-offs, as long as $f$ contains points of odd period. This was not known before as the trade-offs were becoming increasingly less pronounced (approaching the trivial value 1) when $f$'s period was growing.
    \item Finally, the obtained period-specific depth-width trade-offs are shown to hold against shallow networks equipped with \textit{semi-algebraic units} as defined in~\cite{Telgarsky16} and can be extended to the case of high-dimensional input functions by an appropriate projection.
\end{itemize}

Technically, our improved results are based on a tighter eigenvalue analysis of the dynamical systems arising from the periodic orbits in $f$ and on some new connections between the Lipschitz constant of $f$, its (prime) period, and the growth rate of the oscillatory behaviour of repeatedly composing  $f$ with itself. This latter connection allows us to lower bound the $L^1$ error of shallow (but wide) networks, yielding period-specific depth-width lower bounds.

At a broader perspective, we completely answer a question raised by~\cite{Telgarsky16}, regarding the construction of large families of hard-to-represent functions. Our results are tight, as one can explicitly construct examples of functions that achieve equality in our bounds (see \hyperref[lem:example]{Lemma}~\ref{lem:example}). En route to our results, we unify and extend previous methods for depth separations~\cite{Telgarsky16,ICLR, schmitt2000lower}.

Last but not least, we complement our theoretical findings with experiments on a synthetic data set to validate our obtained $L^1$ bounds and also contextualize the fact that depth can indeed be beneficial for some simple learning tasks involving functions of certain periods.

\subsection{Background on dynamical systems}
Here we give the necessary background from dynamical systems in order to later state our results more formally. From now on, $f:[a,b] \to [a,b]$ is assumed to be continuous.

\paragraph{Periods:}
The notion of a \textit{periodic} point (a generalization of a \textit{fixed} point) will be important:
\begin{definition}
We say $f$ contains period $n$ or has a point of period $n \ge 1$, if there exists a point $x_0 \in [a,b]$ such that\footnote{As usual, $f^{n}(x_0)$ denotes the composition of $f$ with itself $n$ times, evaluated at point $x_0$.}:
$$
	f^{n}(x_0)=x_0 {\text{ \ \ and\ \ \ }} f^{k}(x_0)\neq x_0, \forall\ 1 \le k \le n-1.
$$
In particular, $C := \{x_0,f(x_0), f(f(x_0)),\dots, f^{n-1}(x_0)\}$ has distinct elements (each of which is a point of period $n$) and is called a cycle (or orbit) with period $n$.
\end{definition}

Observe that since $f:[a,b] \to [a,b]$ is continuous, it must have a point of period 1, i.e., a fixed point.

\paragraph{Sharkovsky's Theorem:} Recently,~\cite{ICLR} used the period of $f$ to derive period-specific depth-width trade-offs via Sharkovsky's theorem~\citep{sharkovsky1964coexistence,sharkovsky1965cycles} from dynamical systems that provides restrictions on the allowed periods $f$ can have:

\begin{definition}
Define the following (decreasing) ordering $\triangleright$ called \textit{Sharkovsky's ordering}:
\[
3 \triangleright 5 \triangleright 7\triangleright \ldots \triangleright2\cdot3 \triangleright 2\cdot5 \triangleright 2\cdot7\triangleright\ldots
\]
\[
\ldots \triangleright 2^2\cdot3 \triangleright 2^2\cdot5 \triangleright 2^2\cdot7\triangleright\ldots\triangleright2^3\triangleright2^2\triangleright2\triangleright1
\]
We write $l\triangleright r$ or $r\triangleleft l$ whenever $l$ is to the left of $r$ and this gives a total ordering on the natural numbers.
\end{definition}

Observe that the number 3 is the largest according to this ordering. Sharkovsky showed a surprising and elegant result about his ordering: it describes which numbers can be periods for a
continuous map on an interval; allowed periods must be a suffix of his ordering:

\begin{theorem}[Sharkovsky's Theorem]
\label{thm:sharkovsky}
If $f$ contains period $n$ and $n \triangleright n'$, then $f$ also contains period $n'$.
\end{theorem}

According to Sharkovsky's Theorem, 3 is the maximum period, so one important and easy-to-remember corollary is that period 3 implies \textit{all} periods.\footnote{On a historical plot twist, this special case was proved a decade later by James Yorke and Tien-Yien Li, in their seminal paper called ``Period Three Implies Chaos''~\citep{li1975period}; this is a celebrated result that introduced the term ``chaos'' as used in Mathematics (chaos theory).}

We finally need the definition of a \textit{prime period} for $f$:
\begin{definition}[Prime period]\label{def:prime} A function $f$ has prime period $n$ as long as it contains period $n$, but has no periods greater than $n$, according to the Sharkovsky $\triangleright$  ordering.
\end{definition}

Notice that for $f(x)=1-x$, its prime period is 2, since $f(f(x))=1-(1-x)=x$, which implies that it also has fixed point $(f(\tfrac12)=\tfrac12)$.

\subsection{Classification, $L^1$ and $L^\infty$ errors}
\cite{Telgarsky16} proved that $f$ can be the output of a deep NN, for which any function $g$ belonging to a family of shallow, yet extremely wide NNs, will incur high approximation error. He used the most satisfying measure for \textit{lower} bounding the approximation error between $f$ and $g$, which was the $L^1$ error. We say $L^1$ is satisfying, because if the $L^1$ distance between two functions is large, then certainly there are sets of positive measure in the domain where they differ. Just to make the point clear, if $L^\infty$ was used, it wouldn't imply good depth separations, since $L^\infty$ is extremely sensitive even to single point differences. Of course, the situation gets reversed if instead the goal is to obtain distance upper bounds, for which $L^\infty$ is the most desirable. On the other hand, the classification error used in~\citep{ICLR,telgarsky15} (for exact definition, see \hyperref[sec:prelim]{Section}~\ref{sec:prelim}) is a much weaker notion of approximation, that does not seem appropriate for comparing continuous functions, since $f$ and $g$ can have large classification distance, yet still be the same, almost everywhere (i.e., their $L^1$ is arbitrarily close to zero). An explanation for this is depicted in \hyperref[fig:errors]{Figure}~\ref{fig:errors}.

\begin{figure}[ht]
\vskip 0.2in
\begin{center}
\centerline{\includegraphics[width=0.7\columnwidth]{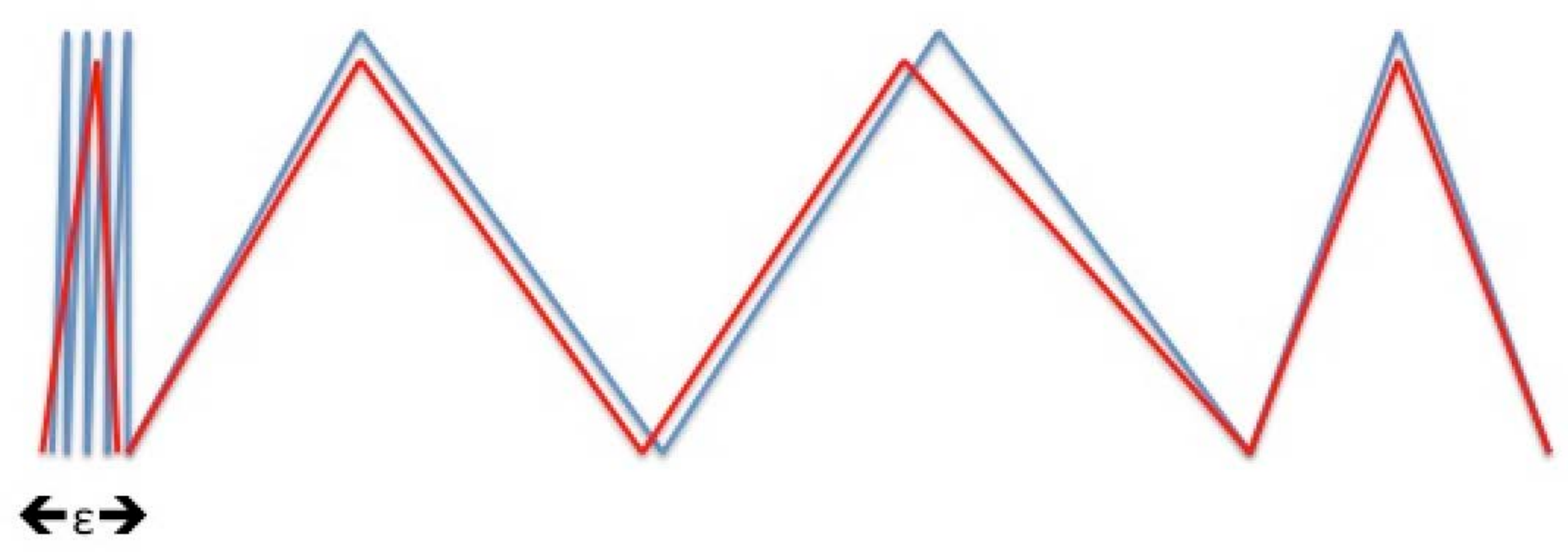}}
\caption{A comparison between two functions that agree almost everywhere. The $L^1$ error is small, however the classification error can be artificially inflated, hence leading to separations, based on unsatisfactory notions of approximation. Similarly, one should not rely on using $L^\infty$ error to get separation results, which is also large in this example.}
\label{fig:errors}
\end{center}
\vskip -0.2in
\end{figure}

To get his $L^1$ bound, Telgarsky presented a simple and highly symmetric construction based on the triangle (or tent) map, which can be thought of as a combination of just two ReLUs. Later he used it to argue that repeated compositions of this map with itself (equivalently, concatenating layers one after the other) yield highly oscillatory outputs. Since functions $g$ generated by shallow networks cannot possibly have so rapidly changing oscillations, he relied on symmetries due to the triangle map and he estimated areas where the two functions differ in order to get a lower bound between $g$ and the triangle compositions.

However, here we can no longer use the specific tent map, since we generalize the constructions based only on the periods of the functions; hence all the symmetries and regularities used to derive the $L^1$ bound are gone. For us, the challenge will be to bound the $L^1$ error based on the periods. For example, in a special case of our result, when the function $f$ has period 3, only 3 values of the function are known on 3 points in the domain.  Can one use such limited information to bound the $L^1$ error against shallow NNs? The natural question that arises is the following:
\begin{center}
\textit{Is it possible to obtain period-specific depth-width trade-offs based on $L^1$ error instead of classification error using only information about the periods?}
\end{center}
Surprisingly, the answer is yes and at a high level, we show that the oscillations arising by function compositions are not pathologically concentrated only on ``tiny'' sets in the domain. Specifically, we carry this out by exploiting some new connections between the prime period of $f$, its Lipschitz constant and the growth rate of the number of oscillations when taking compositions with itself.

\subsection{Periods, Lipschitz constant and Oscillations}
A byproduct of our analysis will be that given two ``design'' parameters for a function $f$, its prime period and Lipschitz constant, we will be able to construct hard-to-represent functions with these parameters, or say it is impossible. This gives a better understanding between those apparently unrelated quantities. To do this we rely on the oscillations that appear after composing $f$ with itself multiple times, as the underlying thread connecting the above notions.

We can show that the Lipschitz constant always dominates the growth rate of the number of oscillations as we take compositions of $f$ with itself. Moreover, whenever its Lipschitz constant matches this growth rate, we can prove that its repeated compositions cannot be approximated by shallow NNs (where the depth of the NN is sublinear in the number of compositions). Finally, we can characterize the number of oscillations in terms of the prime period of the function of interest. These findings provide bounds between the three quantities,: prime period, Lipschitz constant and oscillations.

\subsection{Our contributions}
We now have the vocabulary to state and interpret our results. For simplicity, we will give informal statements that hold against ReLU NNs, but everything goes through for semi-algebraic gates and for higher dimensions as well. Our first result connects the periods with the number of oscillations and improves upon the bounds obtained in~\cite{ICLR}.

\begin{theorem}\label{th:intro}
Let $f: [a,b]\to [a,b]$ have odd prime period $p>1$. Then, there exist points $x<y\in[a,b]$, so that the number of oscillations between $x$ and $y$ is $\rho^t$ where $\rho$ is the root greater than one of the polynomial: \[z^{p-1} - z^{p-2} - \sum_{j=0}^{p-3} (-z)^j=0.\]
\end{theorem}

Our second result ties the Lipschitz constant with the depth-width trade-offs under the $L^1$ approximation.
\begin{theorem}\label{lem:intro-match} Let $f: [a,b] \to [a,b]$ be $L$-Lipschitz, and $g$ be any ReLU NN with $u$ units per layer and $l$ layers. Suppose there exist numbers $x,y\in[a,b]$, such that the oscillations of $f^t$ between $x,y$ are $\Theta(\rho^t)$ for some constant $\rho>1$. As long as $L = \rho$, then for any NN $g$ that has width-depth such that $(2u)^l\le \rho^t$, we get the desired $L^1$-separation:
\[\min_g \int_{a}^b |f^t(z) - g(z)|dz \geq c(x,y) > 0,\] where $c(x,y)$ depends on $x,y$ but \emph{not} on $t$.
\end{theorem}

The above theorem implies depth separations, since if the depth of the ``shallow'' network $g$ is $l=o(t)$, then even exponential width $u$ will not suffice for a good approximation.

Given the above understanding regarding the Lipschitz constant, the periods and the number of oscillations, it is now easy to construct infinite families of functions that are tight in the sense that they achieve the depth-width trade-offs bounds promised by our theorem for any period $p$ (see \hyperref[lem:example]{Lemma}~\ref{lem:example}).

Observe that the largest root of the polynomial in the statement of \hyperref[th:intro]{Theorem}~\ref{th:intro} is always larger than $\sqrt2$. This implies a sharp transition for the depth-width trade-offs, since the oscillations growth rate will be at least $\sqrt2$, whenever $f$ contains an odd period. Previous results, only acquired a base in the exponent that would approach 1, as the (odd) period $p$ increased, and it is known that if $f$ does not contain odd factors in its prime period, then the oscillations can grow only polynomially quickly~\citep{ICLR}.

Finally, in our experimental section we give a simple regression task based on periodic functions, that validates our obtained $L^1$ bound and we also demonstrate how the error drops as we increase the depth.

\medskip

\section{Preliminaries}\label{sec:prelim}

In this section we provide some important definitions and facts that will be used for the proofs of our main results. First we define the notion of crossings/oscillations.

\begin{definition}[Crossings/Oscillations]\label{def:crossings} A continuous function $f: [a,b]\to [a,b]$ crosses the interval $[x,y]$ with $x,y \in [a,b]$ if there exist $c,d\in[a,b]$, such that $f(c)=x$ and $f(d)=y$. Moreover we denote $\textrm{C}_{x,y}(f)$ the number of times $f$ crosses $[x,y]$. It holds $\textrm{C}_{x,y}(f)=t$ if there exist numbers $a_1,b_1<a_2,b_2<\ldots<a_t,b_t$ in $[a,b]$ so that $f(a_i) = x$ and $f(b_i)=y$ for all $1 \leq i \leq t$.
\end{definition}

We next mention the definition of covering relation between two intervals $I_1, I_2$. This notion is crucial because as we shall see later, it enables us to define a graph and analyze the spectral properties of its adjacency matrix. Bounding  the spectral norm of the adjacency matrix from below will enable us to give lower bounds on the number of crossings/oscillations.

\begin{definition}[Covering relation]\label{def:covering}
Let $f$ be a function and $I_1 , I_2$ be two closed intervals. We say that $I_1$ covers $I_2$ under $f$, denoted by $I_1 \xrightarrow{f} I_2$ whenever $I_2 \subseteq f(I_1).$
\end{definition}

We conclude this section with the definition of $L^1$ and classification error.

\begin{definition}[$L^1$ error]
For two functions $f,g: [a,b]\to [a,b]$, their $L^1$ distance is: $$
\int_{[a,b]}|f(x)-g(x)| dx.$$
\end{definition}

\begin{definition}[Classification error]
If we specify a collection of $n$ points $(x_i,y_i)_{i=1}^n$ with $y_i \in\{0,1\}$, one can define the classification error of a function $g$ to be: $$\mathcal{R}(g) = \frac{1}{n} \sum_{i=1}^n \textbf{1}[\tilde{g}(x_i) \neq y_i],$$ where
$\tilde{g}(z)=\textbf{1}[g(z)\ge v]$ is the thresholded value of $g$ based on some chosen threshold $v$ (e.g., $v$ could be $\tfrac12$).
\end{definition}

\section{Lipschitz constant and Oscillations}
In this section, we provide characterizations of continuous $L$-Lipschitz functions $f:[a,b] \to [a,b]$, the compositions of which cannot be approximated (in terms of $L^1$ error) by shallow NNs. For the rest of this section, we assume that there exist $x,y\in[a,b]$ with $x<y$, such that the number of oscillations is:
$\textrm{C}_{x,y}(f^t)\ge C\rho^t, \forall t\in\mathbb{N}$, where $\rho$ is a constant greater than one (we shall call $\rho$ the growth rate of the oscillations) and $C$ is some positive constant.

The lemma below formalizes the idea that a highly oscillatory function needs to have large Lipschitz constant, by showing that $L \geq \rho$.

\begin{lemma}[Lower bound on $L$]\label{lem:lowerbound} Let $f :[a,b] \to [a,b]$ be as above. It holds that $L^t$ is at least $C' \rho^t$, where $C'$ is another positive constant.
\end{lemma}
\begin{proof}
Without loss of generality let $n$ be even, and let it denote the number of oscillations between $x,y$ of the function $f^t$, i.e., $n \geq C \rho^t$. Let $a\leq a_0< a_1<\ldots<a_n \leq b$ be the points such that $f^t(a_{2r+1}) = x$ and $f^t(a_{2r})=y$ for $0 \leq r \leq \tfrac n2$. Since $f^t$ has Lipschitz constant $L^t$, it holds that
$\frac{y-x}{L^t} \leq a_{i+1} - a_{i}$ for all $0 \leq i \leq n-1$. By adding these inequalities, we get a telescoping sum, and we conclude:
\[
\frac{n\cdot(y-x)}{L^t} \leq  \sum_{i=0}^{n-1} (a_{i+1} - a_i) = a_n - a_1 \leq b - a.
\]

Therefore $L^t \geq \frac{(y-x) n}{b-a} \geq C' \rho^t$, where $C'=C\cdot\frac{(y-x)}{b-a}$ is a positive constant.
\end{proof}

An immediate corollary of \hyperref[lem:lowerbound]{Lemma}~\ref{lem:lowerbound} is $L\geq \rho$ as desired.
\subsection{Lipschitz matches oscillations rate for $L^1$ error}
In this section, we give sufficient conditions for a class of functions $f$, so that it cannot be approximated (in $L^1$ sense) by shallow ReLU NNs, and we will later extend it to semi-algebraic gates. The key statement is that the Lipschitz constant of such a function \textit{should match} the growth rate of the number of oscillations.

Assume that $g: [a,b] \to [a,b]$ is a neural network with $l$ layers and $u$ nodes (activations) per layer. It is known that a ReLU NN with $u$ ReLU's per layer and with $l$ layers is piecewise affine with at most $(2u)^l$ pieces~\citep{telgarsky15}.

From now on, let $h := f^{t}$ for ease of presentation. We define as $\tilde{h}(z) = \textbf{1}[h(z) \geq \frac{x+y}{2}]$ and $\tilde{g}(z) = \textbf{1}[g(z) \geq \frac{x+y}{2}]$ for some chosen values of $x,y\in[a,b]$ to be defined later (as we shall see, $x,y$ are just points for which $h$ oscillates between them). Let $\mathcal{I}_{h,x,y}, \mathcal{I}_{g,x,y}$ be the partition of $[a,b]$, where $\tilde{h} , \tilde{g}$ are piecewise constant respectively. We also define $\tilde{\mathcal{J}}_{h,x,y} \subseteq \mathcal{I}_{h,x,y}$ the collection of intervals with the extra assumption that there exists $w$ in each of them such that $h(w)=y$ or $h(w)=x$. Finally define a maximal (in cardinality) sub-collection of intervals $\mathcal{J}_{h,x,y} \subseteq \tilde{\mathcal{J}}_{h,x,y}$ in such a way that if $U_1, U_2$ are consecutive intervals in $\tilde{\mathcal{J}}_{h,x,y}$, the image $h(U_1)$ contains $x$ and the image of $h(U_2)$ contains $y$ (or vice-versa), that is there is an alternation between $x,y$.
It follows~\citep{telgarsky15} that
\begin{equation}
\frac{1}{|\mathcal{J}_{h,x,y}|} \sum_{U \in \mathcal{J}_{h,x,y}}\textbf{1}[\forall z\in U. \tilde{h}(z) \neq \tilde{g}(z)] \geq \frac{1}{2}\left(1 - 2\frac{|\mathcal{I}_{g,x,y}|}{|\mathcal{J}_{h,x,y}|}\right).
\end{equation}

Moreover, one can show the following claim for any interval $U \in \mathcal{J}_{h,x,y}$, and we will use this later:
\begin{claim}\label{claim:mikro}
Let $U \in \mathcal{J}_{h,x,y}$, then \[\int_{U} \left|h(z) - \frac{x+y}{2}\right| dz \geq \tfrac{(y-x)^2}{8L^{t}}.\]
\end{claim}
\begin{proof}
Firstly, observe that $h$ is Lipschitz with constant $L^{t}$ by definition and without loss of generality let's assume $x<y$. In what follows, we make use of the intermediate value theorem for continuous functions.

First we consider the case where there exists a $w \in U$ such that $h(w)=y$.

Let $c<d$, with $c,d \in U$ so that $h(c) = h(d) = \frac{y}{2} + \frac{x+y}{4}$ and $h(z)\geq \frac{y}{2} + \frac{x+y}{4}$ for $z \in [c,d]$ and $w \in [c,d]$ with $f(w)=y$.
It is clear that
\[\int_{U} \left|h(z) - \frac{x+y}{2}\right|dz \geq  \frac{y-x}{4} (d-c).\]
Finally, by the fact that $h$ is Lipschitz with constant $L^{t}$, it follows that $(d-c) = (d-w)+(w-c) \geq \frac{y-x}{4L^{t}}+\frac{y-x}{4L^{t}} = \frac{y-x}{2L^{t}}$. The claim for the case there exists $w \in U$ with $h(w)=y$ follows by substitution. See also \hyperref[fig:lemma_figs]{Figure~\ref{fig:lemma_figs}}.

Similarly, we consider the case in which there exists a $w \in U$ such that $h(w)=x$. Let $c<d$, with $c,d \in U$ so that $h(c) = h(d) = \frac{x}{2} + \frac{x+y}{4}$ and $h(z)\leq \frac{x}{2} + \frac{x+y}{4}$ for $z \in [c,d]$ and $w \in [c,d]$ with $f(w)=x$.
It is clear that
\[\int_{U} \left|h(z) - \frac{x+y}{2}\right|dz \geq  \frac{y-x}{4} (d-c).\]
Again using the fact that $h$ is Lipschitz with constant $L^{t}$, it follows that $(d-c) = (d-w)+(w-c) \geq \frac{y-x}{4L^{t}}+\frac{y-x}{4L^{t}} = \frac{y-x}{2L^{t}}$. The claim for the case in which there exists a $w \in U$ such that $h(w)=x$ follows by substitution.
\end{proof}

\begin{figure}
	\centering
	\subfigure[An illustration of the area computed by the integral in the case of a piecewise concave function (red curve). However note that composition of concave functions may not be concave.]{\includegraphics[width = 2.8in]{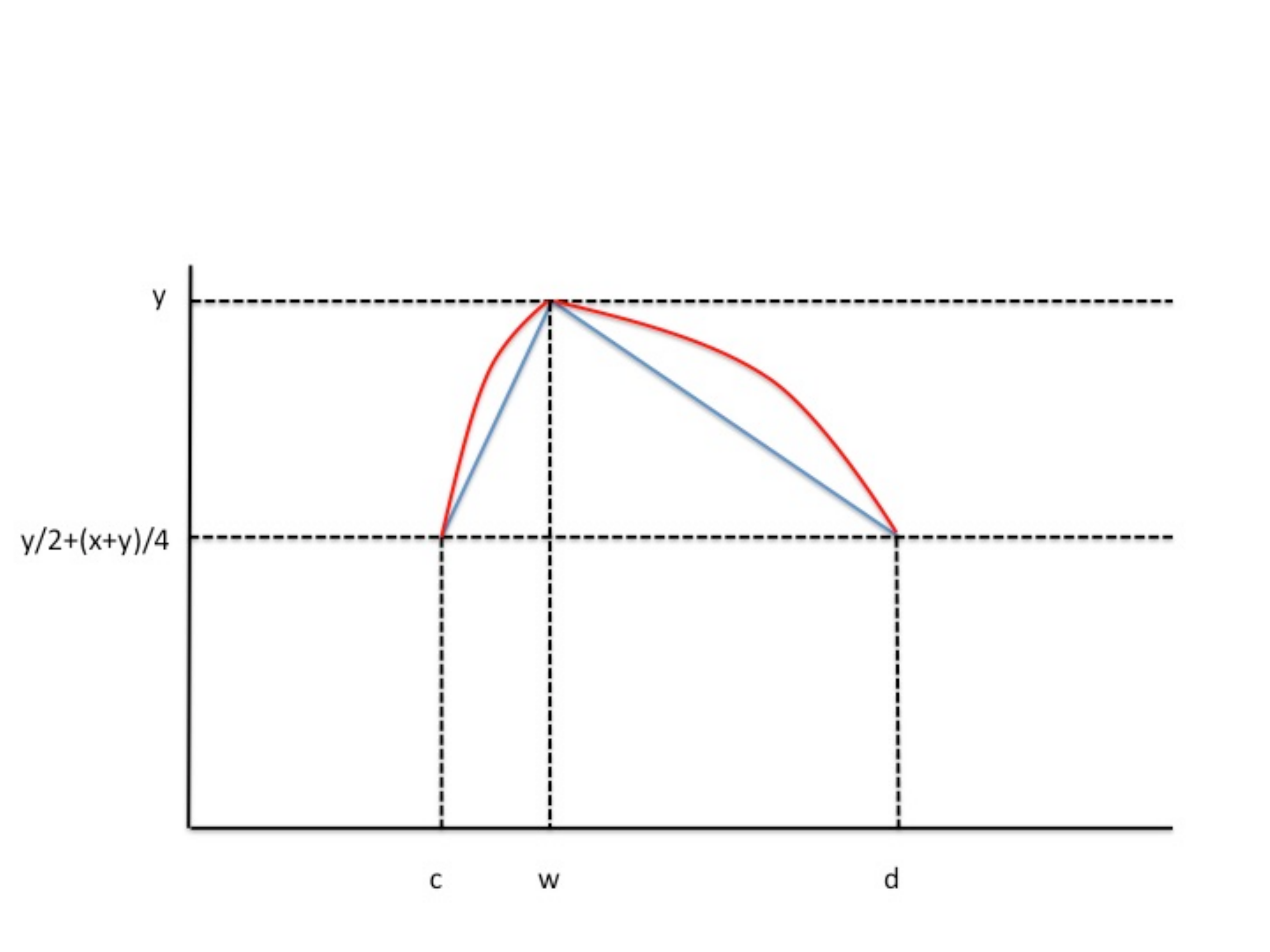}}
	\hspace{2cm}
	\subfigure[The lower  bound for the integral would not necessarily hold in the case of piecewise convex functions, as the area could become arbitrarily small.]{\includegraphics[width = 2.8in]{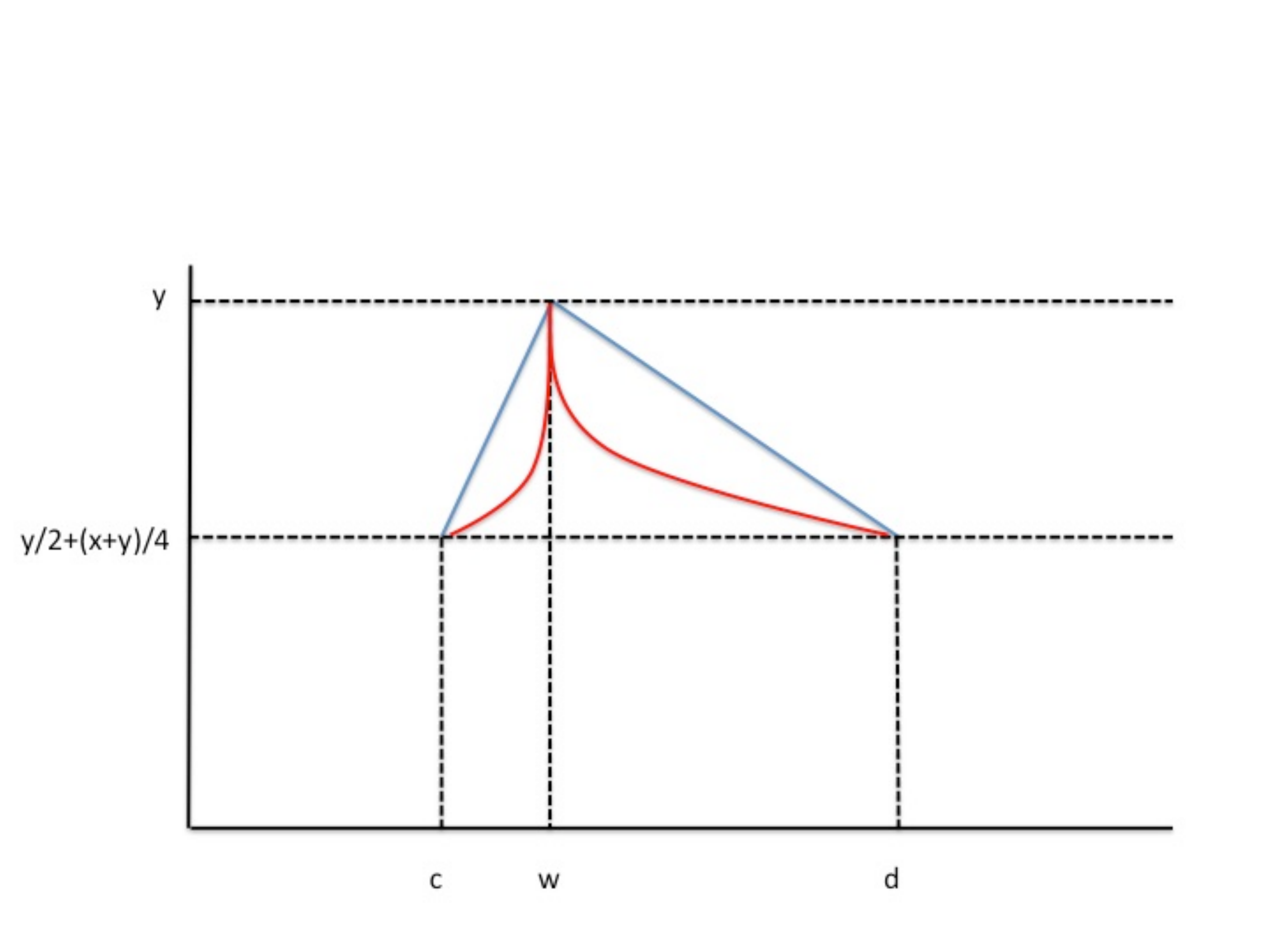}}
	\caption[]
	{\small Graphs demonstrating the proof described in \hyperref[claim:mikro]{Claim \ref{claim:mikro}} with respect to the Lipschitz constant. The lower bound on the area computed by the integral would hold even if the repeated compositions created piecewise concave functions (left); however, for convex functions no guarantee on the area of the triangle (hence no $L^1$ separation) can be derived (right).}
	\label{fig:lemma_figs}
\end{figure}

As mentioned in the beginning, a sufficient condition for a function $f$ to be hard-to-represent, is that its Lipschitz constant should be exactly equal to the base in the growth rate of the number of oscillations. In Telgrasky's paper, the function used was the tent map that has Lipschitz constant equal to 2, and the oscillations growth rate under repeated compositions also had growth rate $2$. This is not a coincidence and here we generalize this observation.

\begin{theorem}[Lipschitz matches oscillations]\label{lem:match} Let $f: [a,b] \to [a,b]$ be $L$-Lipschitz, and $g$ be any ReLU NN with $u$ units per layer and $l$ layers. Suppose there exist $x,y$ such that the oscillations of $f^t$ between $x,y$ are $\Theta(\rho^t)$ for some constant $\rho>1$. As long as $L = \rho$ (by \hyperref[lem:lowerbound]{Lemma}~\ref{lem:lowerbound} we already know that $L \geq \rho$) and $(2u)^l \leq \frac{\rho^t}{8}$, then we get the desired $L^1$-separation:
\[\min_g \int_{a}^b |f^t(z) - g(z)|dz \geq c(x,y) > 0,\] where $c(x,y)$ depends on $x,y$ but \emph{not} on $t$.
\end{theorem}
\begin{proof}
We will prove a lower bound for the $L^1$ distance between $h := f^t$ and an arbitrary $g$ from the aforementioned family of NNs with $(2u)^l \leq \frac{\rho^t}{8}$.
\begin{align*}
&\int_{a}^b \left|h(z)-g(z)\right|dz = \sum_{U \in \mathcal{I}_{h,x,y}} \int_U \left|h(z)-g(z)\right|dz
\\&\geq \sum_{U \in \mathcal{J}_{h,x,y}} \int_U \left|h(z)-g(z)\right|dz
\\&\geq \sum_{U \in \mathcal{J}_{h,x,y}} \int_U \left|h(z) - \frac{x+y}{2}\right|\textbf{1}[\forall z\in U. \tilde{h}(z) \neq \tilde{g}(z)]dz
\\&\geq \frac{|\mathcal{J}_{h,x,y}|(y-x)^2}{16L^{t}} \left(1 - 2\frac{|\mathcal{I}_{g,x,y}|}{|\mathcal{J}_{h,x,y}|}\right).
\end{align*}
It is clear that $|\mathcal{J}_{h,x,y}|$ is at least the number of crossings $\textrm{C}(f^{t})$, hence we conclude that $|\mathcal{J}_{h,x,y}|$ is $\Theta(\rho^t)$. It follows that
\[\int_{a}^b \left|h(z)-g(z)\right|dz \geq \frac{\left(\frac{\rho}{L}\right)^t(y-x)^2}{16}\left(1 - 2\frac{|\mathcal{I}_{g,x,y}|}{\rho^t}\right),\] where $$|\mathcal{I}_{g,x,y}| \leq (2u)^l.$$
Thus, as long as $(2u)^l \leq \frac{\rho^t}{8}$, and since $L = \rho$, we conclude that \[\int_{a}^b \left|h(z)-g(z)\right|dz \geq \frac{(y-x)^2}{32}.\]
\end{proof}

\begin{remark}[Larger Lipschitz:] Observe that if we didn't require that $L=\rho$ and instead $L>\rho$, no meaningful $L^1$ guarantee could be derived since the term $(\frac{\rho}{L})^t$ would shrink for large $t$ (see also \hyperref[fig:lemma_figs]{Figure~\ref{fig:lemma_figs}}).
\end{remark}

\begin{remark}[Semi-algebraic activation units:] Our results can be easily generalized for the general class of semi-algebraic units (see \cite{Telgarsky16} for definitions). The idea works as follows: Any neural network that has activation units that are semi-algebraic, it is piecewise polynomial, therefore piecewise monotone (the pieces depend on the degree of the polynomial, which in turn depends on the specifications of the activation units). Therefore, the function $\tilde{g}$ (as defined above) is piecewise constant and defines a partition of the domain $[a,b]$. The crucial observation is that the size of this partition is bounded by a number that depends exponentially on the number of layers (i.e., layers appear in the exponent) and polynomially on the number of units per layer (i.e., width is in the base). Finally, our results can be applied for the multivariate case. As in~\cite{Telgarsky16}, we handle this case by first choosing an affine map $\mu : \mathbb{R}\to\mathbb{R}^d$ (meaning $\mu (z) = \kappa z + \nu$) and considering functions $f^t \circ \mu$.
\end{remark}

\subsection{Periodicity and Lipschitz constant}
In this subsection, we improve the result of \cite{ICLR}, by showing that functions $f$ of odd period $p>1$ have points $x<y$ so that the number of oscillations between $x$ and $y$ is $\rho^t$, where $\rho$ is the root greater than one of the polynomial equation \[z^{p-1} - z^{p-2} - \sum_{j=0}^{p-3} (-z)^j= \frac{z^p - 2z^{p-2}- 1}{z+1}=0.\]

This consists an improvement from the previous result in \cite{ICLR} that states that the growth rate of the oscillations of compositions of a function with $p$-periodic point is the root greater than one of the polynomial $z^{p-1} - z^{p-2} -1 =0$ (observe that the two aforementioned polynomials coincide for $p=3$). This is true because if $\rho, \rho'$ are the roots of $\frac{z^p - 2z^{p-2}- 1}{z+1}$ and $z^{p-1} - z^{p-2} -1$, then $\rho >\rho'$, unless $p=3$ for which we have $\rho = \rho'$. This gives better depth-width trade-offs for any value of the (odd) period.

Moreover, if $\rho$ is the Lipschitz constant of $f$, then \hyperref[lem:match]{Lemma}~\ref{lem:match} applies and any shallow neural network $g$ (with $(2u)^l \leq \frac{\rho^t}{8}$) has $L^1$ distance bounded away from zero for any number of compositions of $f$.

We first need the following structural lemma~\cite{german}.
\begin{lemma}[Monotonicity \cite{german}]\label{lem:monotonicity} Let $p>1$ be an odd number and consider $f:[a,b] \to [a,b]$ with prime period $p$. Then there exists a cycle of period $p$ with points $\{x_1,...,x_p\}$ such that
\[x_p < x_{p-2} < ... < x_3 < x_1 < x_2 < x_4 <...< x_{p-1}.\]
\end{lemma}

This lemma will help us define an appropriate covering relation, to be used later in order to bound the number of oscillations in $f^t$. Towards this goal, we set $I_0 = [x_1,x_2]$, $I_{j} = [x_{2j}, x_{2j+2}]$ for $1\leq j\leq \frac{p-3}{2}$ and $J_{j} = [x_{2j+1},x_{2j-1}]$ for $1 \leq j \leq \frac{p-1}{2}$. From \hyperref[lem:monotonicity]{Lemma}~\ref{lem:monotonicity}, we trivially have the following covering relations.
\begin{corollary}\label{cor:covering}
It holds that
\begin{itemize}
\item $I_0 \to I_0 \cup J_1.$
\item $I_j \to J_{j+1}$, for $1 \leq j \leq \frac{p-3}{2}$.
\item $J_j \to I_j$, for $1 \leq j \leq \frac{p-3}{2}$.
\item $J_{\frac{p-1}{2}} \to I_0 \cup I_1 \cup ... \cup I_{\frac{p-3}{2}}$.
\end{itemize}
\end{corollary}

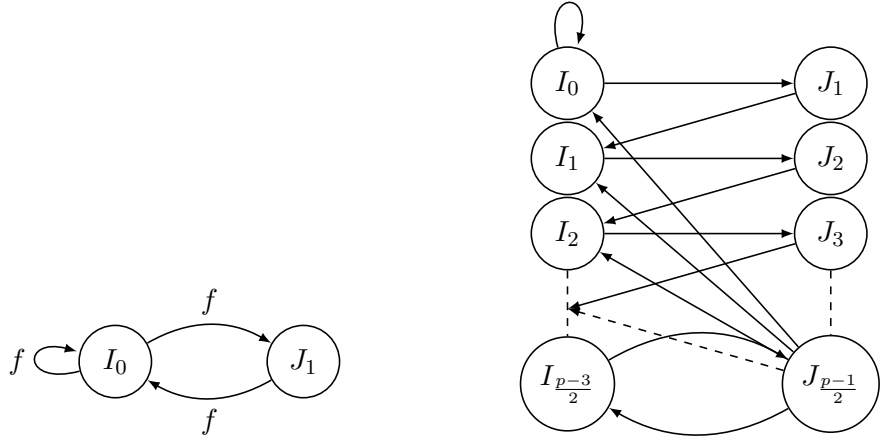
\begin{figure}
	\centering
	\subfigure[Covering relations as given by Corollary \ref{cor:covering} for a cycle of period three.]
		{
		\begin{tikzpicture}[node distance=1.5cm and 1cm,
		>=latex,auto,
		semithick,  
		initial distance=1cm,
		every initial by arrow/.style={->}
		]
		
		\node[state] (q0) {$I_0$};
		\node[state, right of=q0,  xshift=1cm] (q1) {$J_1$};
		
		\path[->] (q0) edge [bend left, above] node {\tt $f$} (q1);
		\path[->] (q1) edge [bend left, below] node {\tt $f$} (q0);
		
		\draw (q0) edge [loop left] node {\tt $f$} (q0);
		\end{tikzpicture}
		
	}
		\label{fig:cov_3}
	\hspace{2cm}
	\subfigure[Covering relations as given by Corollary \ref{cor:covering} for a cycle of an odd period $p$ greater than three. The directed dashed edge from  $J_{\frac{p-1}{2}}$ indicates that the edge goes to every node between $I_2$ and $I_{\frac{p-3}{2}}$. We omit $f$ from the arrows for ease of presentation.]
		{
		\begin{tikzpicture}[node distance=1.5cm and 1cm,
		>=latex,auto,
		semithick,  
		initial distance=1cm,
		every initial by arrow/.style={->}
		]
		
		\node[state] (q0) {$I_0$};
		\node[state, right of=q0,  xshift=2cm,yshift=0cm] (q1) {$J_1$};
		\node[state, left of=q1,  xshift=-2cm,yshift=-1cm] (q2) {$I_1$};
		\node[state, right of=q2,  xshift=2cm,yshift=0cm] (q3) {$J_2$};
		\node[state, left of=q3,  xshift=-2cm,yshift=-1cm] (q4) {$I_2$};
		\node[state, right of=q4,  xshift=2cm,yshift=0cm] (q5) {$J_3$};
		\node[state, left of=q5,  xshift=-2cm,yshift=-2cm] (q6) {$I_{\frac{p-3}{2}}$};
		\node[state, right of=q6,  xshift=2cm,yshift=0cm] (q7) {$J_{\frac{p-1}{2}}$};
		\coordinate[left of=q5, xshift=-2cm,yshift=-1cm] (d5);
		
		\path[->] (q0) edge [] node {} (q1);
		\path[->] (q1) edge [] node {} (q2);
		\path[->] (q2) edge [] node {} (q3);
		\path[->] (q3) edge [] node {} (q4);
		\path[->] (q4) edge [] node {} (q5);
		\path[-] (q5) edge [dashed] node {} (q7);
		\path[-] (q4) edge [dashed] node {} (q6);
		\path[->] (q7) edge [bend left, above] node {} (q6);
		\path[->] (q7) edge [] node {} (q4);
		\path[->] (q7) edge [] node {} (q2);
		\path[->] (q7) edge [] node {} (q0);
		\path[->] (q5) edge [] node {} (d5);
		\path[->] (q6) edge [bend left, below] node {} (q7);
		\path[->] (q7) edge [dashed] node {} (d5);

		\draw (q0) edge [loop above] node {} (q0);
		\end{tikzpicture}
	}
		\label{fig:cov_gen}
	\caption[]
	{\small The covering relations of intervals from Corollary \ref{cor:covering} are shown here. Note the existence of more directed edges when the period is odd and greater than 3, compared to the graph used to quantify the growth in \cite{ICLR}. This allows us to obtain improved bounds for $p>3$.}
	\label{fig:covering}
\end{figure}

Let $A \in \mathbb{R}^{(p-1) \times (p-1)}$ be the adjacency matrix of the covering relation graph above (the intervals denote the nodes of the graph):
\begin{align}
\begin{cases}
	A_{ji} = 1\; \text{, if}\; i=0,j=0 \\
	A_{ji} = 1\; \text{, if}\; j=i+\frac{p-1}{2} \;\text{and}\; 0\leq i \leq p-2 \\
	A_{ji} = 1\; \text{, if}\; j=i+1-\frac{p-1}{2} \;\text{and}\; \frac{p-1}{2}\leq i \leq p-2 \\
	A_{ji} = 1\; \text{, if}\; i=p-2 \;\text{and}\; 0 \leq j \leq \frac{p-3}{2} \\
    A_{ji} = 0\; \text{, otherwise}
\end{cases}
\end{align}

We define $\delta^t$ which is in $\mathbb{N}^{p-1}$ to keep track of how many times $f^t$ crosses specific intervals. The coordinate $\delta^t_i$ captures the number of times $f^t$ crosses interval $I_i$ for $i \leq \frac{p-3}{2}$ and the coordinate $\delta^t_{i+\frac{p-1}{2}}$ captures the number of times $f^t$ crosses interval $J_i$. We get that

\begin{align}
\left(\begin{array}{c}
\delta^{t+1}_0 \\
\delta^{t+1}_1 \\
\vdots \\
\delta^{t+1}_{p-1}
\end{array}\right)
\geq
A
\left(\begin{array}{c}
\delta^{t}_0 \\
\delta^{t}_1 \\
\vdots \\
\delta^{t}_{p-1}
\end{array}\right),
\end{align}
where $\delta^0= (1,\dots,1)$ (all ones vector).

\begin{claim}
The characteristic polynomial of $A^{\top}$ has the same roots as
\begin{equation}\label{eq:polynomial}
\pi_{p}(\lambda) = \lambda^{p-1} - \lambda^{p-2} - \sum_{j=0}^{p-3} (-\lambda)^j  = \frac{\lambda^p - 2\lambda^{p-2}- 1}{\lambda+1}.
\end{equation}
\end{claim}
\begin{proof} For $p=3$, the desired equation holds, since the matrix $A^{\top}$ becomes just \[A^{\top} = \left(
	\begin{array}{cc}1 & 1\\1 & 0 \end{array} \right),\] with characteristic polynomial $(\lambda-1)\lambda - 1 = \lambda^2 - \lambda -1$. Let $I$ denote the identity matrix of size $(p-1)\times (p-1)$. Assume $p \geq 5$. We consider the matrix:
	\[A^{\top} - \lambda I  = \left(
	\begin{array}{cc}
	A_{11} & A_{12}\\
	A_{21} & A_{22}\\
	\end{array}\right)
	\]
	where
	$A_{11} :=\left(\begin{array}{cccccccc}
	1-\lambda & 0 & 0 & 0 & 0 & \ldots & 0\\
	0 & -\lambda & 0 & 0 & 0 & \ldots & 0\\
	0 & 0 & -\lambda & 0 & 0 & \ldots & 0 \\
	\vdots & \vdots & \vdots & \vdots & \vdots &\vdots & \vdots\\
	0 & 0 & 0 & 0 & 0 & \ldots & -\lambda
	\end{array}\right),$ $A_{12} :=\left(\begin{array}{cccccccc}
	1 & 0 & 0 & 0 & \ldots & 0\\
	0 &1 &0 &0 &\dots &0\\
	0 &0 &1 &0 &\dots &0\\
	\vdots & \vdots & \vdots & \vdots & \vdots & \vdots\\
	0 & 0 &0 &0 &\dots &1
    \end{array}\right),$\\

	$A_{21} :=\left(\begin{array}{cccccccc}
	0 & 1 & 0 & 0 & \ldots & 0 &  0\\
	0 & 0 & 1 & 0 & \ldots & 0 &  0\\
	\vdots & \vdots & \vdots & \vdots & \vdots &\vdots & \vdots\\
	0 & 0 & 0 & 0 & \ldots & 0&  1\\
	1 & 1 & 1 & 1 & \ldots & 1&  1
	\end{array}\right),$ and $A_{22} :=\left(\begin{array}{cccccccc}
    -\lambda & 0 & 0 & \ldots & 0  & 0\\
    0 & -\lambda & 0 & \ldots & 0  & 0\\
    \vdots & \vdots & \vdots & \vdots & \vdots & \vdots\\
    0 & 0 & 0 & \ldots &  -\lambda &0\\
    0 & 0 & 0 & \ldots & 0 & -\lambda
	\end{array}\right).$
	
	Observe that $\lambda = 0$ is not an eigenvalue of the matrix $A^{\top}$. Suppose that $A_{11},A_{12},A_{21},A_{22}$ are the four block submatrices of the matrix above. Using Schur's complement, we get that
    $\textrm{det}(A^{\top} - \lambda I) = \textrm{det}(A_{22}) \times \textrm{det}(A_{11} - A_{12}A^{-1}_{22}A_{21})$, where $\textrm{det}(A_{22}) = (-\lambda)^{\frac{p-1}{2}}$ and
    \begin{equation*}
    \begin{array}{cc}
    \lambda^{\frac{p-1}{2}}\textrm{det}(A_{11} - A_{12}A^{-1}_{22}A_{21}) = \\
	\left(
	\begin{array}{cccccccccccccccc}
	\lambda-\lambda^2 & 1 & 0 & 0 & 0 & \ldots & 0 \\
	0 & -\lambda^2 & 1 & 0 & 0 & \ldots & 0\\
	0 & 0 & -\lambda^2 & 1 & 0 & \ldots & 0\\
	\vdots & \vdots & \vdots & \vdots & \vdots &\vdots & \vdots\\
	1 & 1 & 1 & 1 & \ldots &1 & -\lambda^2+1\\
	\end{array}
	\right).
	\end{array}
	\end{equation*}

    We can multiply the first row by $\tfrac{1}{\lambda(\lambda-1)}$, the second row by $\tfrac{1}{\lambda^2} + \tfrac{1}{\lambda^2 \lambda(\lambda-1)}$, the third row by $\tfrac{1}{\lambda^2}+\tfrac{1}{\lambda^4} + \tfrac{1}{\lambda^4 \lambda(\lambda-1)}$,\dots, the $i$-th row by $\sum_{j=1}^{i-1} \tfrac{1}{\lambda^{2j}}+\tfrac{1}{\lambda^{2(i-1)}\cdot  \lambda(\lambda-1)} $ (and so on) and add them to the last row. Let $B$ be the resulting matrix:
	\[B  = \left(
	\begin{array}{cccccccccccccccc}
	\lambda-\lambda^2 & 1 & 0 & 0 & 0 & \ldots & 0 \\
	0 & -\lambda^2 & 1 & 0 & 0 & \ldots & 0\\
	0 & 0 & -\lambda^2 & 1 & 0 & \ldots & 0\\
	\vdots & \vdots & \vdots & \vdots & \vdots &\vdots & \vdots\\
	0 & 0 & 0 & 0 & \ldots &0 & K\\
	\end{array}
	\right),\]
	where $K=-\lambda^2+1+\sum_{j=1}^{\frac{p-5}{2}} \tfrac{1}{\lambda^{2j}}+\tfrac{1}{\lambda^{p-5}\cdot  \lambda(\lambda-1)}.$
	It is clear that the equation $\textrm{det}(B)=0$ has the same roots as $\textrm{det}(A^{\top} - \lambda I)=0$. Since $B$ is an upper triangular matrix, it follows that
	\begin{equation*}
	    \begin{array}{cc}

	\textrm{\ \ \ \ \ det}(B)= (-1)^{\frac{p-5}{2}} \lambda(\lambda - 1) \lambda^{p-5} \cdot \left(-\lambda^2+1+\sum_{j=1}^{\frac{p-5}{2}} \tfrac{1}{\lambda^{2j}}+\tfrac{1}{\lambda^{p-5}\cdot  \lambda(\lambda-1)}\right).
		\end{array}
		\end{equation*}
	We conclude that the eigenvalues of $A^{\top}$ (and hence of $A$) must be roots of
\begin{align*}
&(\lambda^{p-3} - \lambda^{p-4})\left(1-\lambda^2 +\sum_{j=1}^{\frac{p-5}{2}} \tfrac{1}{\lambda^{2j}} \right)+1 = -\lambda^{p-1} + \lambda^{p-2} + \lambda^{p-3} - \lambda^{p-4} + \sum_{j=1}^{\frac{p-5}{2}} \lambda^{p-3-2j} - \lambda^{p-4-2j}+1 \\=& -\lambda^{p-1} + \lambda^{p-2} + \sum_{j=0}^{p-3} (-1)^j \lambda^{j} = \frac{-\lambda^{p}+\lambda^{p-2}}{\lambda+1} + \frac{1+\lambda^{p-2}}{\lambda+1} = \frac{-\lambda^p +2\lambda^{p-2}+1}{\lambda+1},
\end{align*}
and the claim follows.
\end{proof}
The following corollary establishes a connection between the growth rate of the oscillations of compositions of function $f$ with its prime period. Also, we establish a universal sharp threshold phenomenon demonstrating that the width needs to grow at a rate at least as large as $\sqrt{2}$, as long as the function contains an odd period (this is in contrast with previous depth separation results where the growth rate converges to one as the period $p$ goes to infinity).

\begin{corollary} Let $f:[a,b] \to [a,b]$ be a continuous function with prime odd period $p>1$. There exist $x,y$ such that the number of oscillations between $x,y$ of $f^t$ is $\Theta(\rho_p^t)$ where $\rho_{p}$ is the positive root greater than one of $q_p(\lambda):=\lambda^{p}-2\lambda^{p-2} -1 = 0$. Moreover, $\rho_p$ is decreasing in $p$ and $\rho_p > \sqrt{2}$, for all $p$.
\end{corollary}
\begin{proof}
We first need to relate the spectral radius with the number of oscillations. We follow the idea from \cite{ICLR} which concludes that $\delta^t_0 \geq \norm{A^t}_{\infty} \geq \textrm{spec}(A^t) = \textrm{spec}(A)^t = \rho_p^t$ (where $\textrm{spec}(A)$ denotes the spectral radius), that is the growth rate of the number of oscillations of compositions of $f$ is at least $\rho_p$.

Assume $1< p$ be an odd number. It suffices to show that $\rho_{p+2} < \rho_p$ (and then use induction). Observe that
$\lambda^{p+2} - 2\lambda^{p} - 1 = \lambda^2 (\lambda^p - 2\lambda^{p-2} - 1) + \lambda^2 -1$. Therefore
\[
 0 = q_{p+2}(\rho_{p+2}) = \rho_{p+2}^2 q_p (\rho_{p+2}) + \rho_{p+2}^2-1,
\]
hence since $\rho_{p+2}>1$ we conclude that $q_p (\rho_{p+2}) <0$. Since $\lim_{\lambda \to \infty} q_p (\lambda) = +\infty$, by Bolzano's theorem it follows that $q_p$ has a root in the interval $(\rho_{p+2},+\infty)$. Thus $\rho_p > \rho_{p+2}$. One can also see that $\sqrt{2}^p - 2\sqrt{2}^{p-2} - 1 = -1<0$ and $2^p - 2\cdot 2^{p-2}-1>0$, thus from Bolzano's again, it follows that $\rho_p>\sqrt{2}$ for all $p$.
\end{proof}

\begin{remark} We note that $\rho_p>\sqrt{2}$ (the growth rate is at least $\sqrt{2}$) whereas the growth rate in \cite{ICLR} was converging to one as $p\to \infty$.
\end{remark}

We now provide tight constructions for a family of functions $f$ that have points of period $p$ (thus the number of oscillations of $t$ compositions of $f$ scales as $\rho_p^t$, i.e., the growth rate is $\rho_p$) and moreover the Lipschitz constant is $\rho_p$. By \hyperref[lem:match]{Theorem}~\ref{lem:match}, this family cannot be approximated by shallow neural networks in $L^1$ sense.
\begin{lemma}\label{lem:example}
Let $p>1$ be an odd number and $\rho_p$ be the largest positive root greater than one of the polynomial $\lambda^{p}-2\lambda^{p-2} - 1 = 0$. The function $f:[-1,1] \to [-1,1]$, defined to be $f(x):=\rho_p |x| - 1$ has Lipschitz constant $\rho_p$ and has period $p$.
\end{lemma}
\begin{proof}
It suffices to show that $f$ has period $p$ (the Lipschitz constant is trivially $\rho_p$). We start from $z_0=0$ and we get $z_t = f(z_{t-1}) = \rho_p|z_{t-1}|-1$ for $1\leq  t \leq p$. Observe that $z_1 = -1$, $z_2 = \rho_p-1>0$. Set $q_i(\lambda) = \frac{\lambda^{i}-2\lambda^{i-2} -1}{\lambda+1}$. First, we shall show that for $t \in \{3,\dots, p-1\}$, we have $z_t \leq 0$ and that $z_{t} = q_{t}(\rho_p)$, whereas for $t$ even, we have $z_{t} = -q_{t-1}(\rho_p) \rho_p -1$ in the interval above.

For $t=3$ we get that $z_3 = \rho_p^2 - \rho_p -1 = q_3(\rho_p) \leq 0$ because we showed $\rho_p$ is decreasing in $p$ and moreover holds $q_3(\rho_3)=0$. Since $z_3 \leq 0$ we get that $z_4 = - \rho_p z_3-1 = q_{3}(\rho_p) \rho_p -1$. Let us show that $z_4 \leq 0$. Observe that $z_4 = -\rho_p^3 +\rho_p^2 + \rho_p -1 = (\rho_p-1)(1-\rho_p^2) < 0$ (since $\rho_p>\sqrt{2}$).

We will use induction. Assume now, that we have the result for some $t$ even, we need to show that $z_{t+1} = q_{t+1}(\rho_p), z_{t+2} = - q_{t+1}(\rho_p) \rho_p-1$ and moreover $z_{t+1},z_{t+2}\leq 0$.

By induction, we have that $z_{t-1},z_{t} \leq 0$ and $z_{t} = -q_{t-1}(\rho_p) \rho_p - 1$, hence $z_{t+1} = -\rho_p (-q_{t-1}(\rho_p) \rho_p - 1) -1 = \frac{\rho_p^{t+1} - 2\rho_p^t - \rho_p^2}{\rho_p+1}+\rho_p -1 = q_{t+1}(\rho_p)$. Since $\rho_p$ is decreasing in $p$ and $q_{t+1}(\rho_{t+1})=0$, we conclude that $z_{t+1}\leq 0$. Since $z_{t+1}\leq 0$, we get that $z_{t+2} = -\rho_p z_{t+1} - 1 = -\rho_p q_{t+1}(\rho_p) - 1$. To finish the claim, it suffices to show that $z_{t+2} \leq 0$. Observe that
\begin{align*}
-\rho_p q_{t+1}(\rho_p) - 1 &= -\rho_p \left(\rho_p^{t} - \rho_p^{t-1} - \sum_{j=0}^{t-2} (-\rho_p)^j\right) - 1\\
& = -\rho_p^{t+1} + \rho_p^{t} - \sum_{j=1}^{t-1} (-\rho_p)^{j} - 1\\& = -2\rho_p^{t+1} + 2\rho_p^t + \frac{q_{t+1}(\rho_p)}{\rho_p+1}.
\end{align*}
The term $-2(\rho_p^{t+1}-\rho_p^t) <0$ (since $\rho_p>1$) and moreover $\frac{q_{t+1}(\rho_p)}{\rho_p+1}\leq 0$ because $\rho_p$ is decreasing in $p$ and $t+1 \leq p-1$. Hence $z_{t+2}\leq 0$ and the induction is complete.

From the above, we conclude that $z_p = -\rho_p z_{p-1} -1 =  q_p (\rho_p) = 0$, thus $z_0,...,z_{p-1}$ form a cycle. If we show that $z_0,...,z_{p-1}$ are distinct, the proof of the lemma follows.

First observe that $q_t(\lambda) = \frac{\lambda^{t}-2\lambda^{t-2}-1}{\lambda+1}$ is strictly increasing in $t$ as long as $\lambda>\sqrt{2}$ (by computing the derivative). Therefore it holds that $z_3 <  z_5<\ldots < z_p =0$ (for all the odd indices) and also $z_1<z_3$. Furthermore, $-\lambda q_t(\lambda) -1$ is decreasing in $t$ for $\lambda>\sqrt{2}$, therefore we conclude $z_4>\ldots>z_{p-1}$ (and also $z_2>0 \geq z_4$).

We will show that $z_3 >z_4$ and finally $z_{p-1} > -1 =z_1$ and the lemma will follow. Recall $z_3 = \rho_p^2 - \rho_p-1$ and $z_4 = -\rho_p^3 + \rho_p^2+\rho_p-1$. Equivalently, we need to show that
$\rho_p^2 - \rho_p-1 > -\rho_p^3 + \rho_p^2+\rho_p-1$ or
$\rho_p^3 -2\rho_p>0$ which holds because $\rho_p >\sqrt{2}$.
Finally $z_{p-1} = -\rho_p z_{p-2} - 1 >-1$ since $z_{p-2}<z_p=0$.

\end{proof}

\subsection{Sensitivity to Lipschitzness and separation examples based on periods}

We end this section with some simple examples that illustrate the behavior of the aforementioned family of functions $f(x):=\rho_p|x| - 1$ for different parameters and the corresponding depth-width trade-offs that can be derived. As a consequence, we will observe how similar-looking functions can actually have vastly different behaviors with regards to oscillations, periods and hence depth separations (see \hyperref[fig:rho_1]{Figure~\ref{fig:rho_1}}).

We consider three regimes. The first regime corresponds to the functions that appear in \hyperref[lem:match]{Lemma \ref{lem:match}}, where $L=\rho_{p}$ and $\rho_p \in [\sqrt{2},\phi]$, where $\phi = \tfrac{1+\sqrt{5}}{2} \approx 1.618$ is the golden ratio. The second regime corresponds to the case when $L > \phi$ and the third regime corresponds to the case when $L < \sqrt{2}$. We can see in \hyperref[fig:rho_2]{Figure~\ref{fig:rho_2}} that the function $f(x):=2|x| - 1$ has period 3 and a Lipschitz constant of $L=2$, while in \hyperref[fig:rho_3]{Figure~\ref{fig:rho_3}}, we can see that the function $f(x):=1.2|x| - 1$, does not have any odd period and $L=1.2$.

\begin{figure}
	\centering
	\subfigure[The regime $ \sqrt2 \le L\le \phi $ with small variations of the slope. Intersection with $y=x$ identifies fixed points.]{\includegraphics[width = 2.8in]{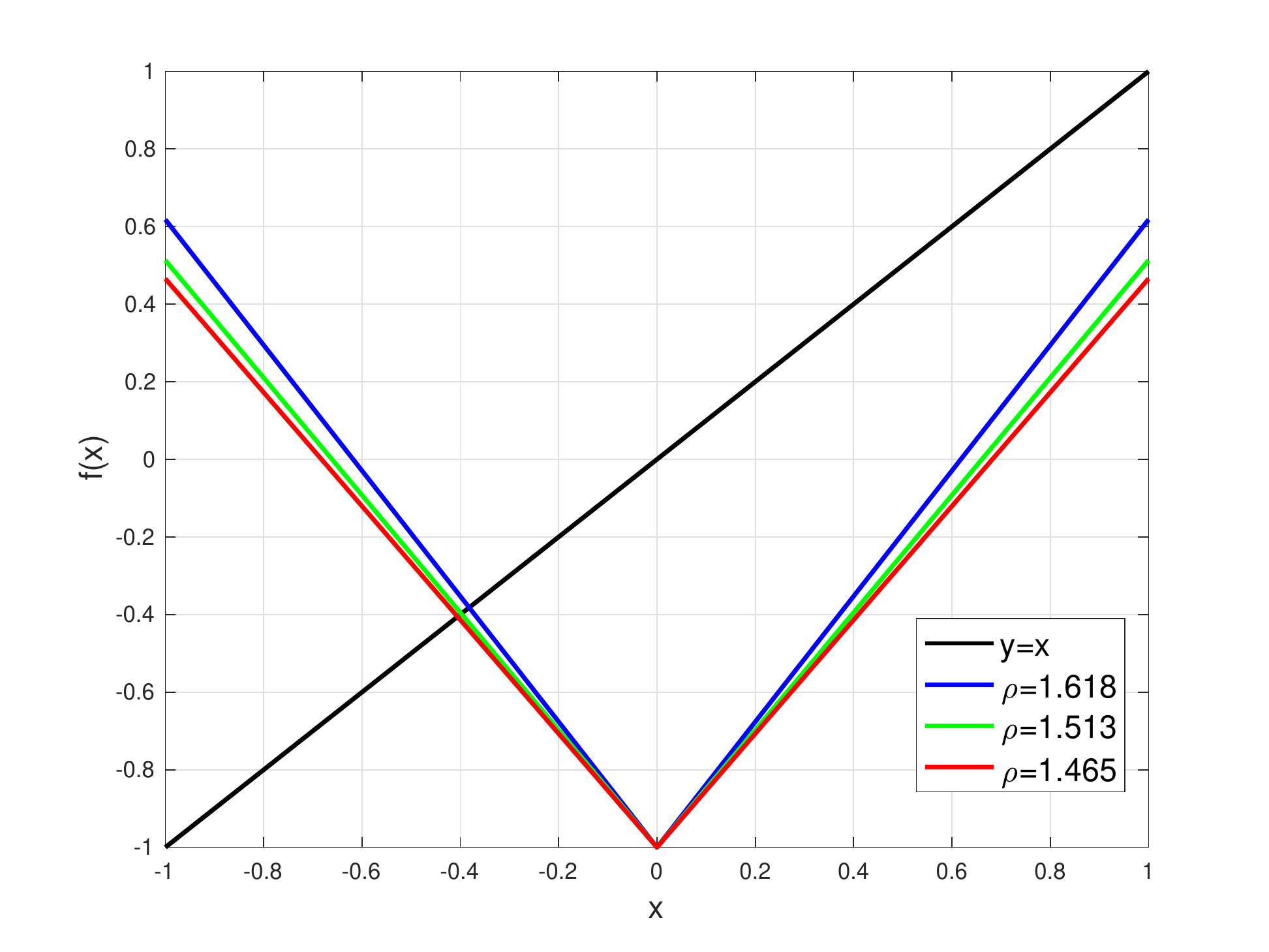}}
	\hspace{2cm}
	\subfigure[Graph of $f^3(x)$ intersected with $y=x$, to identify period 3 points. Only when $L= \phi$, period 3 is present, hence it gives exponential depth-width trade-offs with base $\phi$.]{\includegraphics[width = 2.8in]{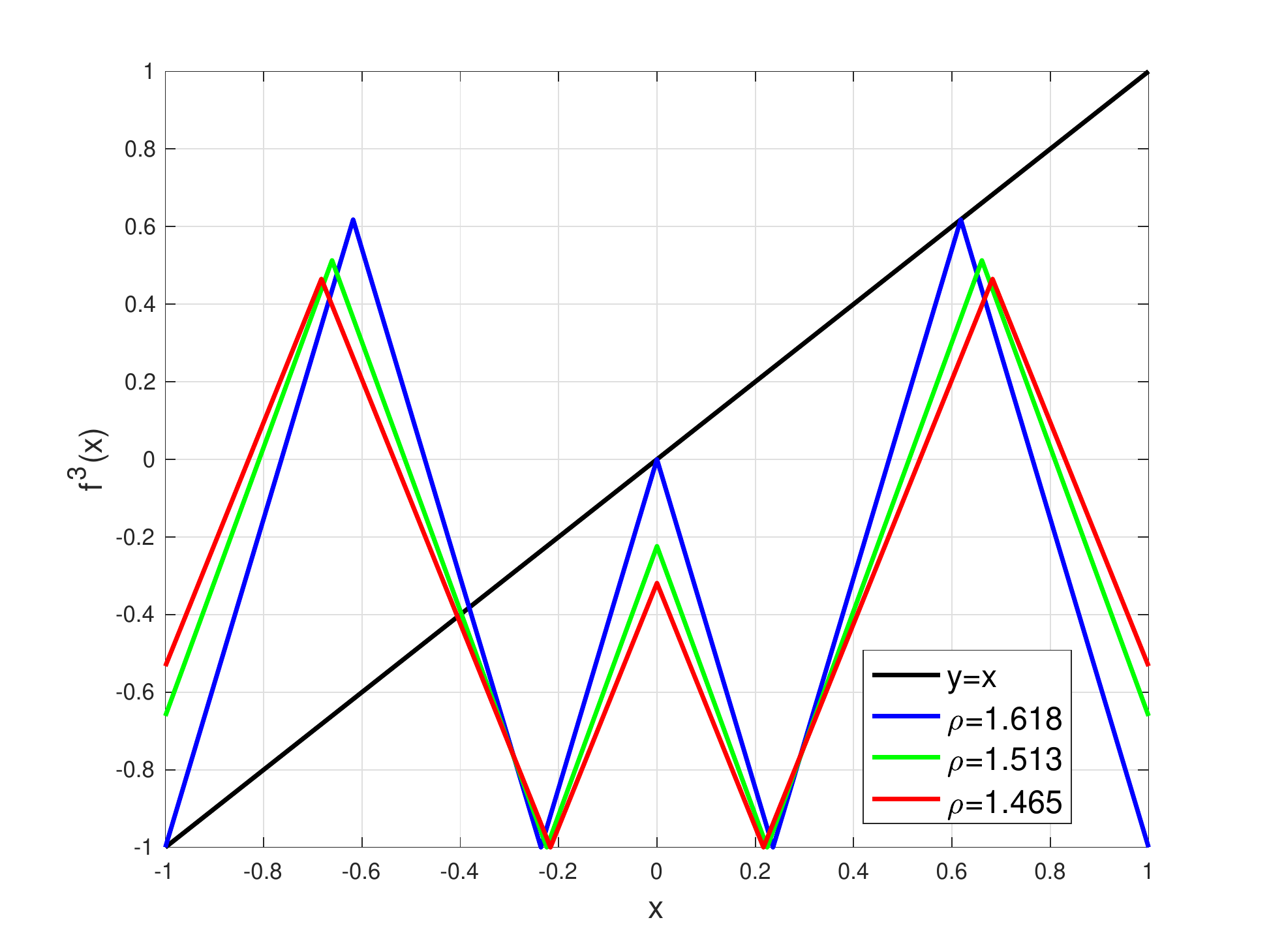}}\\
	\subfigure[Graph of $f^5(x)$ intersected with $y=x$, to identify period 5 points. When $L= 1.513$, period 5 is present, hence it gives trade-offs with base 1.513.]{\includegraphics[width = 2.8in]{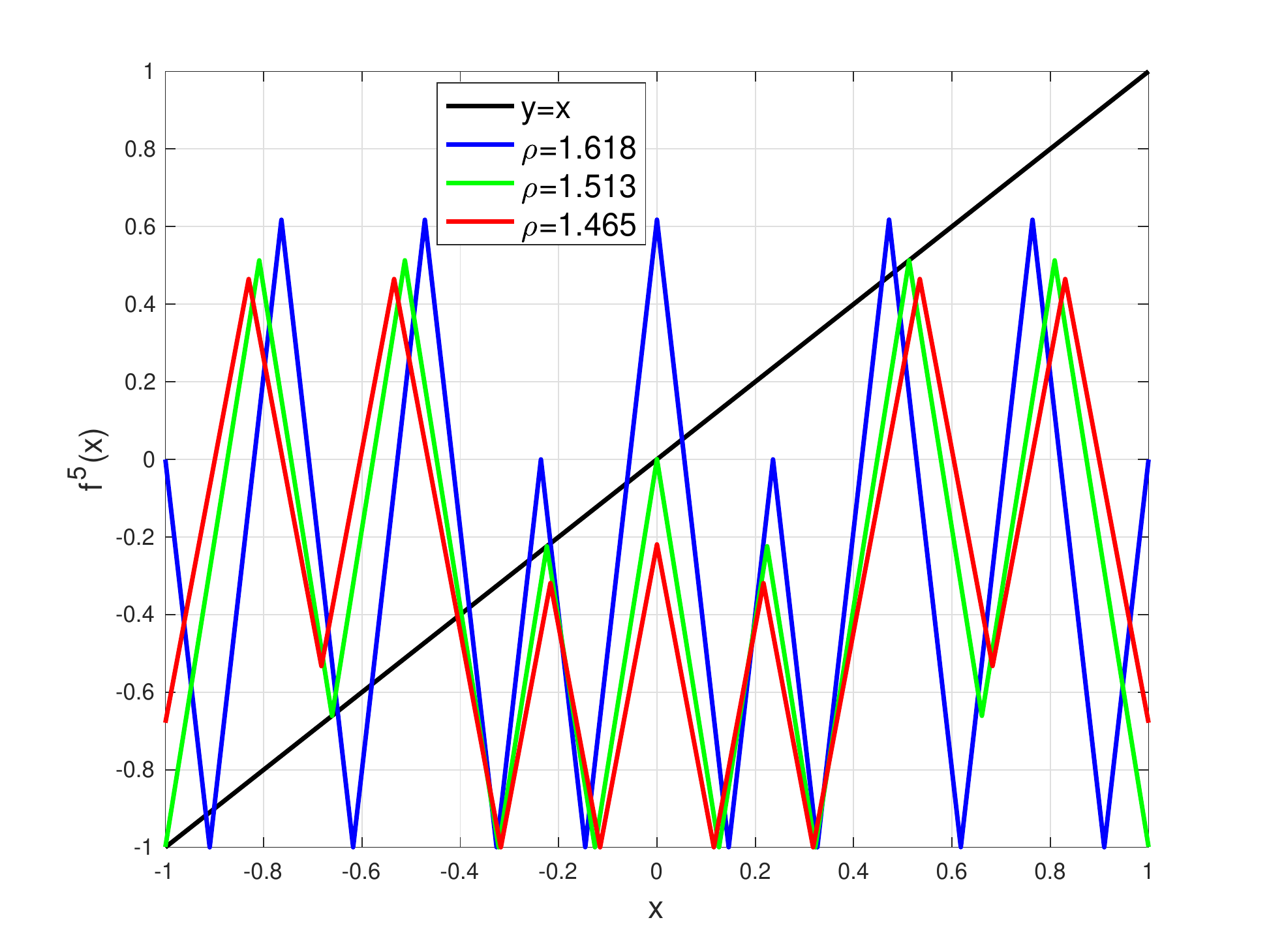}}
	\hspace{2cm}
	\subfigure[Graph of $f^7(x)$ intersected with $y=x$, to identify period 7 points. When $L= 1.465$, period 7 is present, hence it gives trade-offs with base 1.465.]{\includegraphics[width = 2.8in]{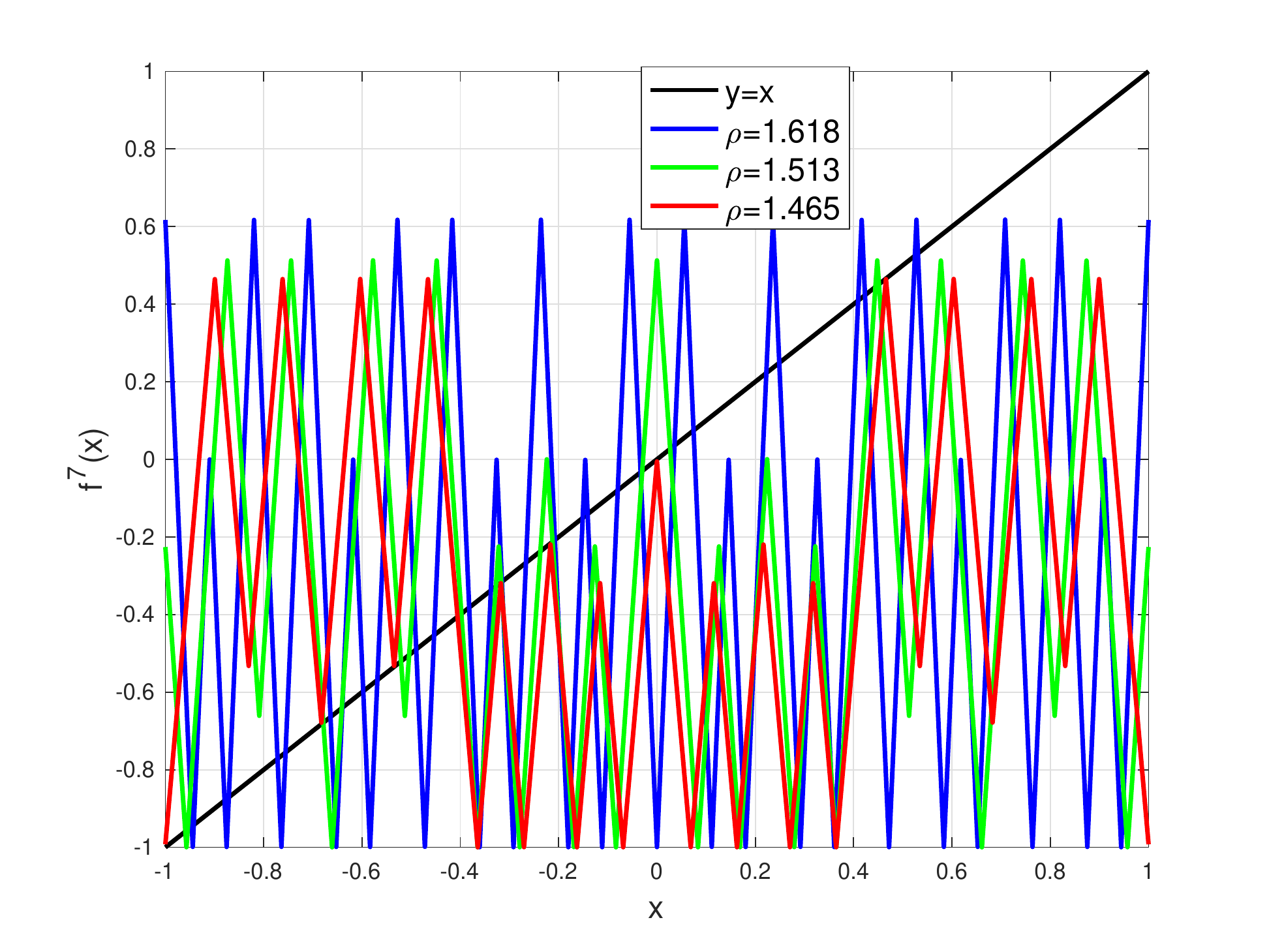}}
	\caption[]
	{\small The family of functions parameterized by $\rho_p$, where $L=\rho_p$ and $\rho=1.618,1.513, 1.465$ correspond to period 3, 5 and 7 respectively. Observe how slight perturbations of the function can lead to different trade-offs.}
	\label{fig:rho_1}
\end{figure}

\begin{figure}
	\centering
	\subfigure[Graph of $f(x)$ intersected with $y=x$, to identify period 1 points.]{\includegraphics[width = 2.8in]{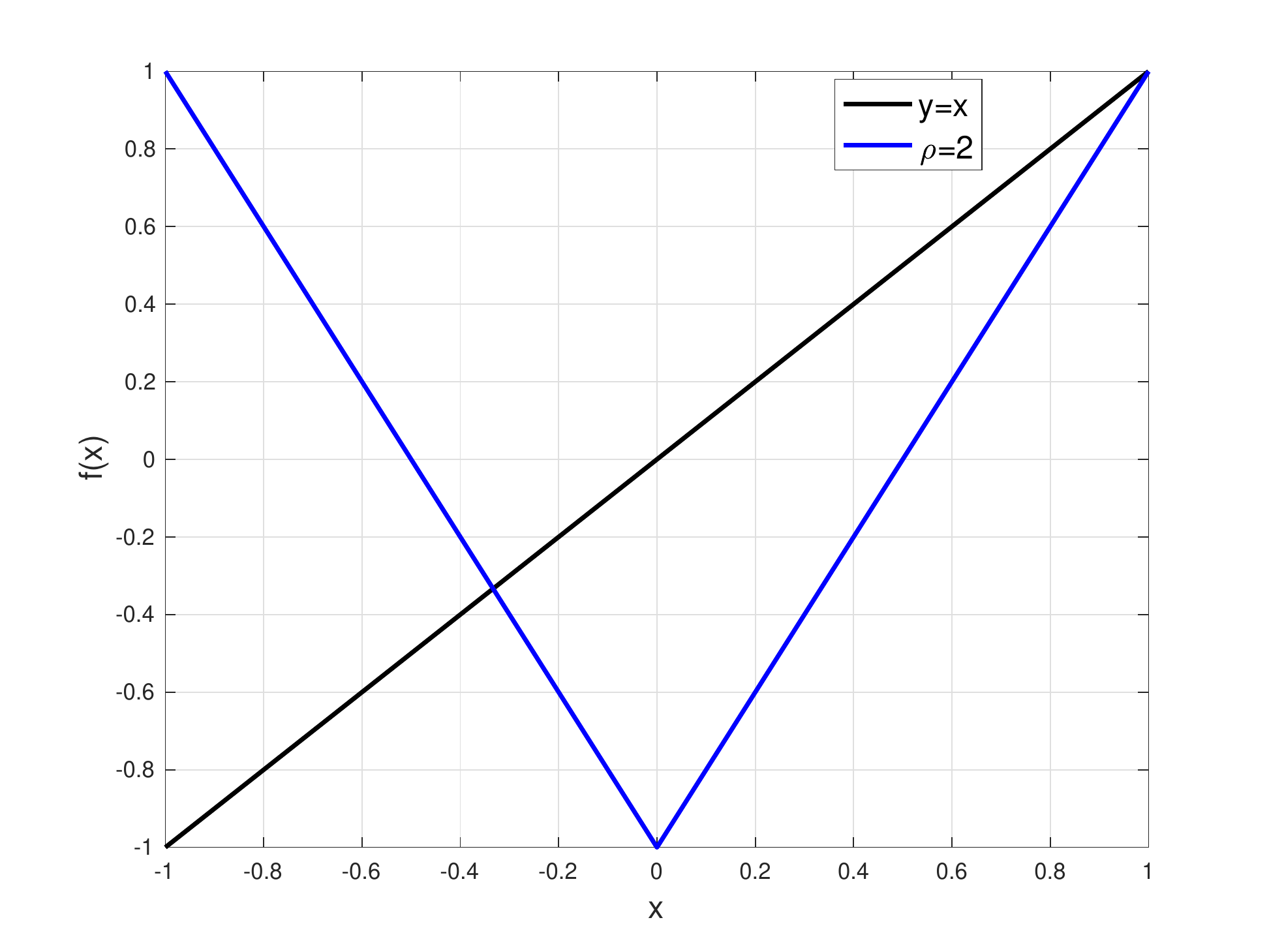}}
	\hspace{2cm}
	\subfigure[Graph of $f^3(x)$ intersected with $y=x$, to identify period 3 points.]{\includegraphics[width = 2.8in]{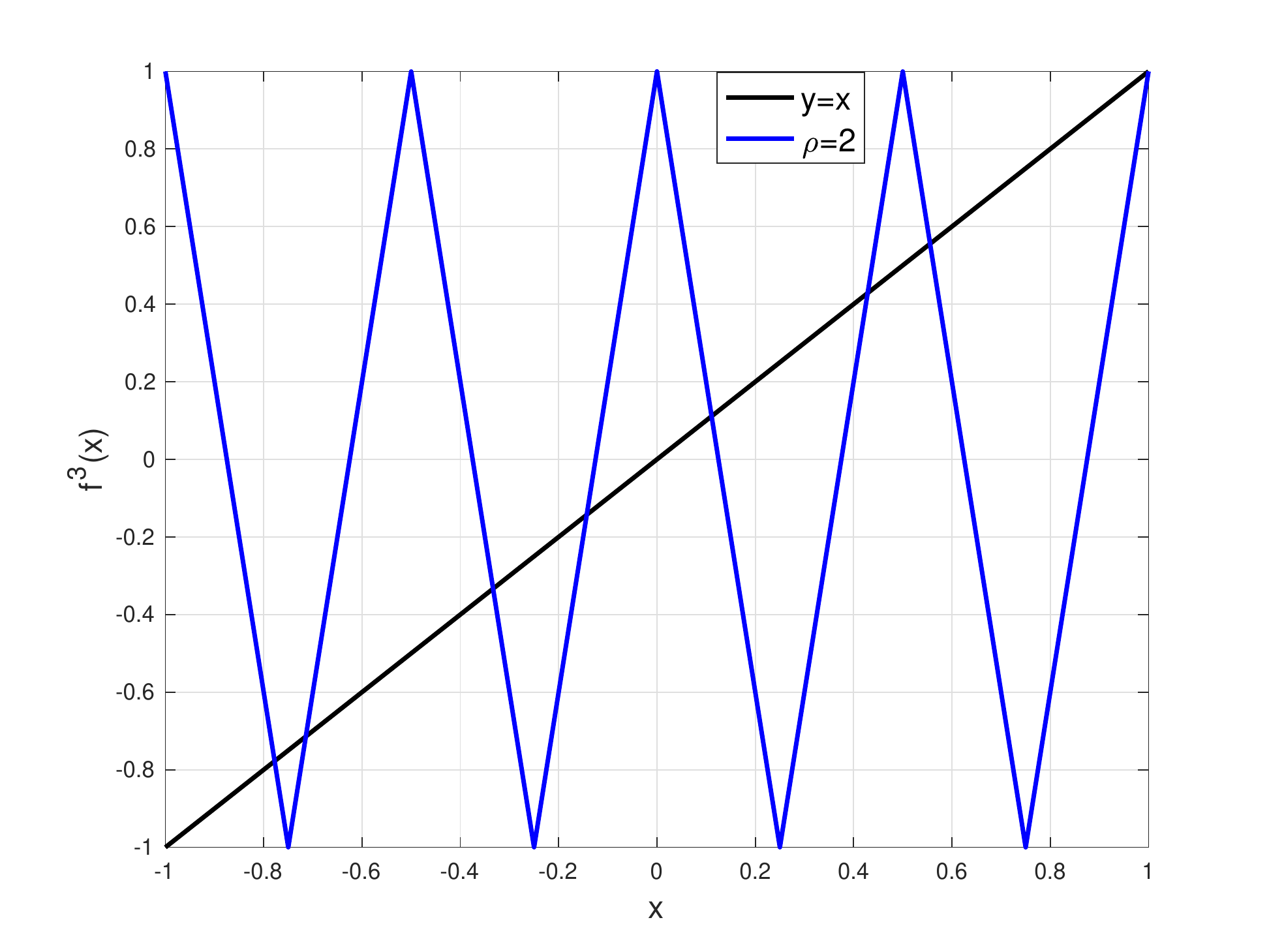}}\\
	\subfigure[Graph of $f^5(x)$ intersected with $y=x$, to identify period 5 points.]{\includegraphics[width = 2.8in]{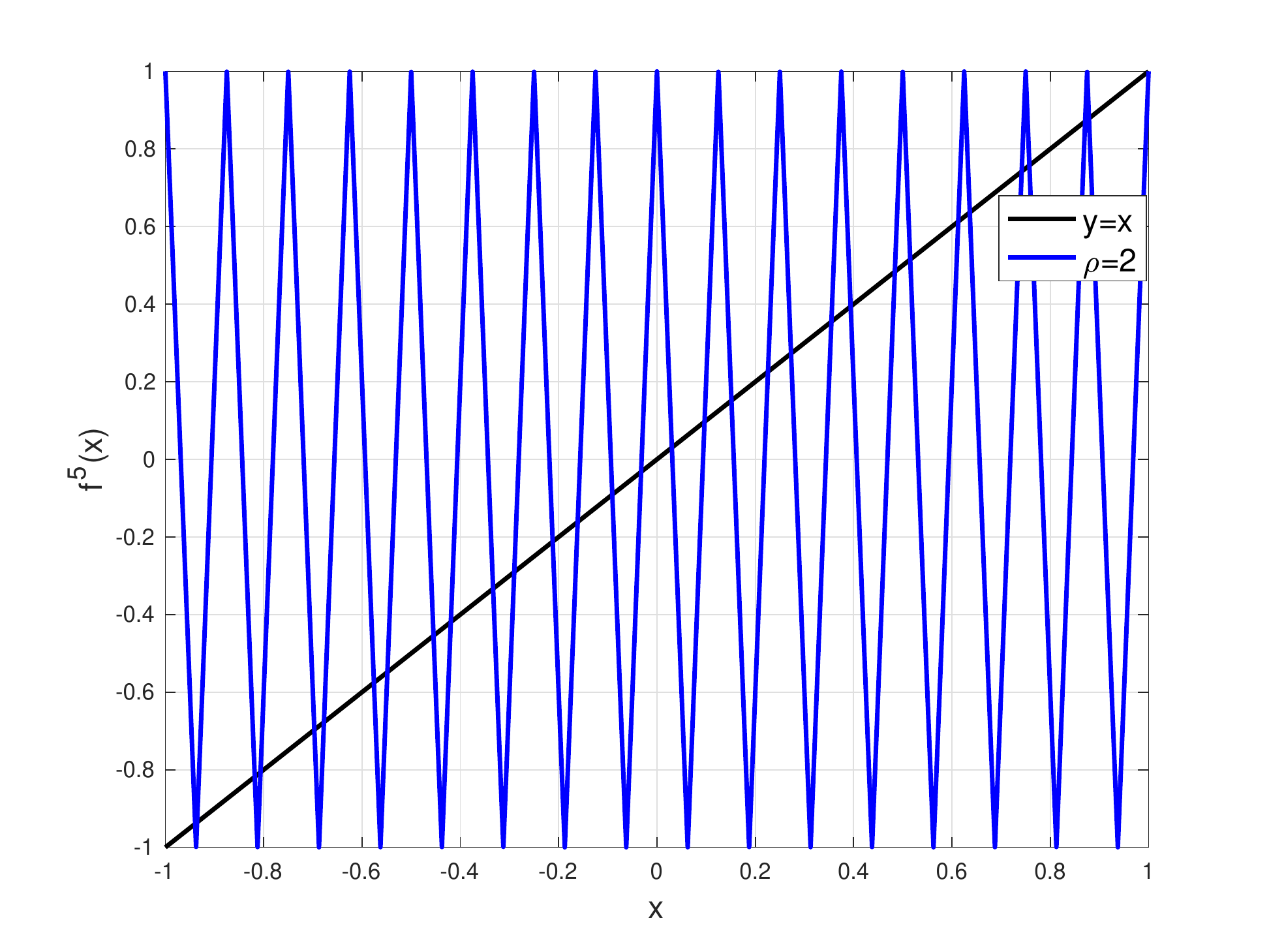}}
	\hspace{2cm}
	\subfigure[Graph of $f^7(x)$ intersected with $y=x$, to identify period 7 points.]{\includegraphics[width = 2.8in]{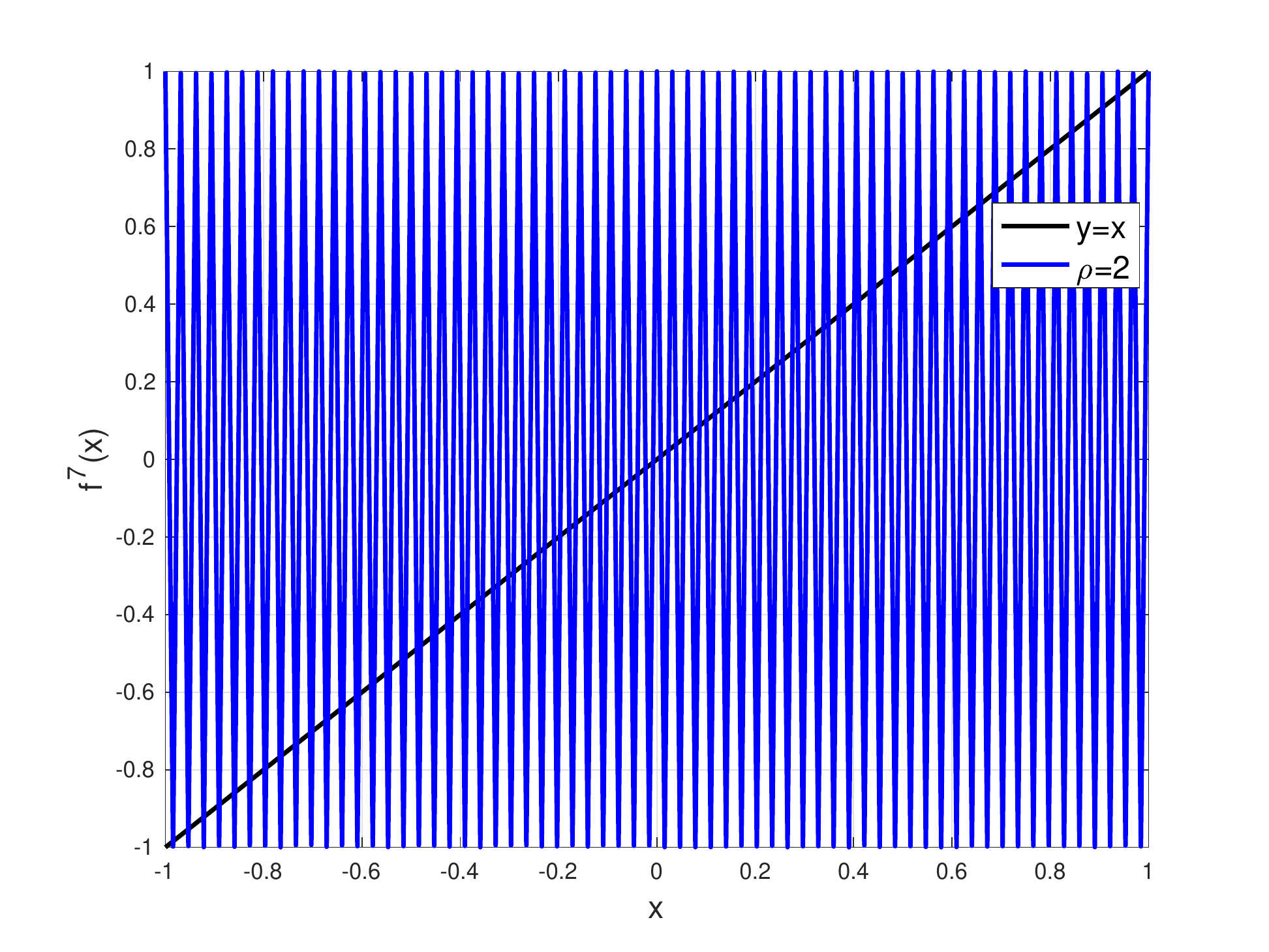}}
	\caption[]
	{\small Here $L=2$, and this function has period 3. However, the growth rate of oscillations is exactly 2 and since we have equality $L=\rho$ we get $L^1$ separations even though the largest root $\rho_3=\phi<2$.}
	\label{fig:rho_2}
\end{figure}

\begin{figure}
	\centering
	\subfigure[Graph of $f(x)$ intersected with $y=x$, to identify period 1 points.]{\includegraphics[width = 2.8in]{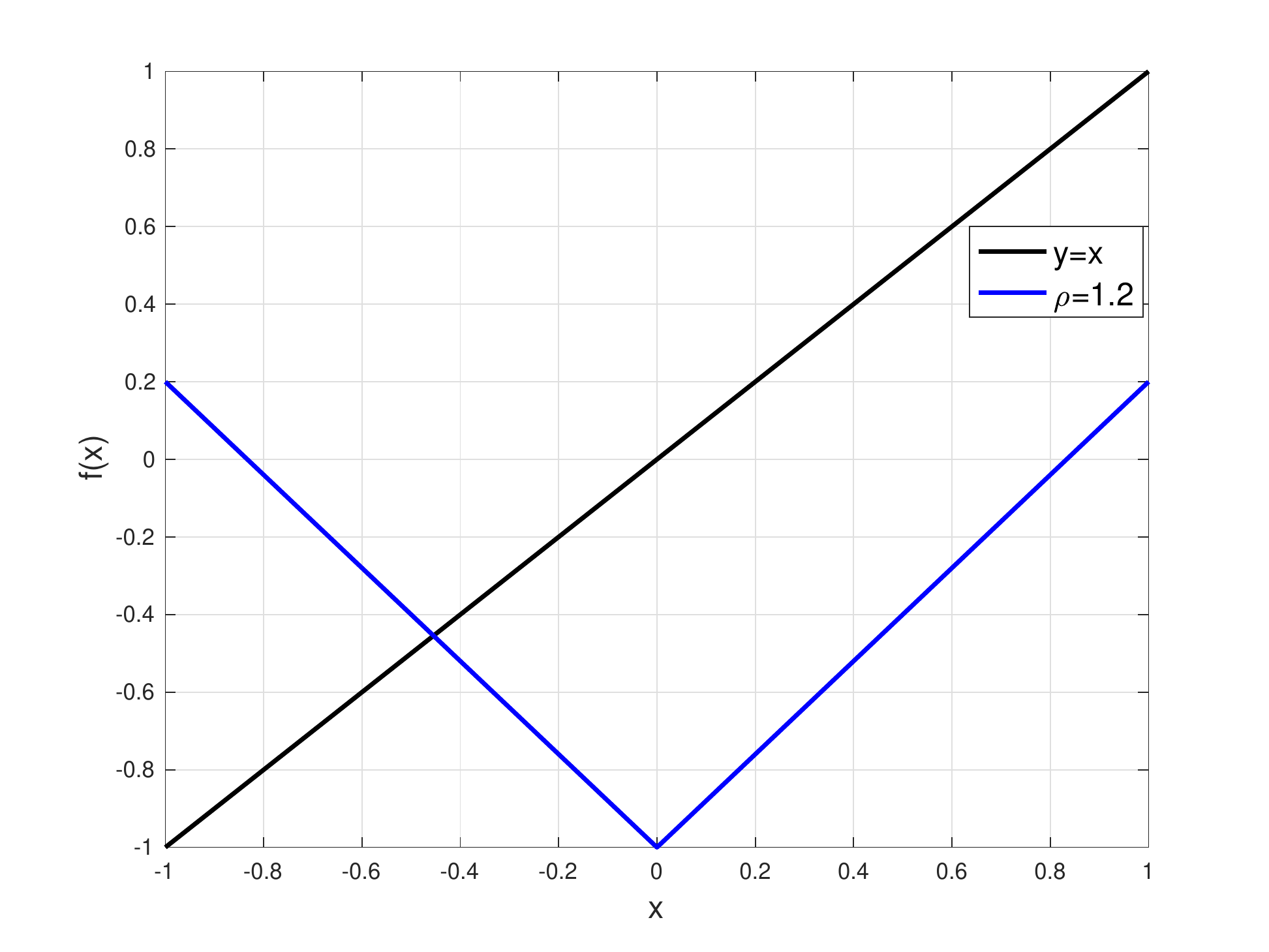}}
	\hspace{2cm}
	\subfigure[Graph of $f^3(x)$ intersected with $y=x$, to identify period 3 points.]{\includegraphics[width=2.8in]{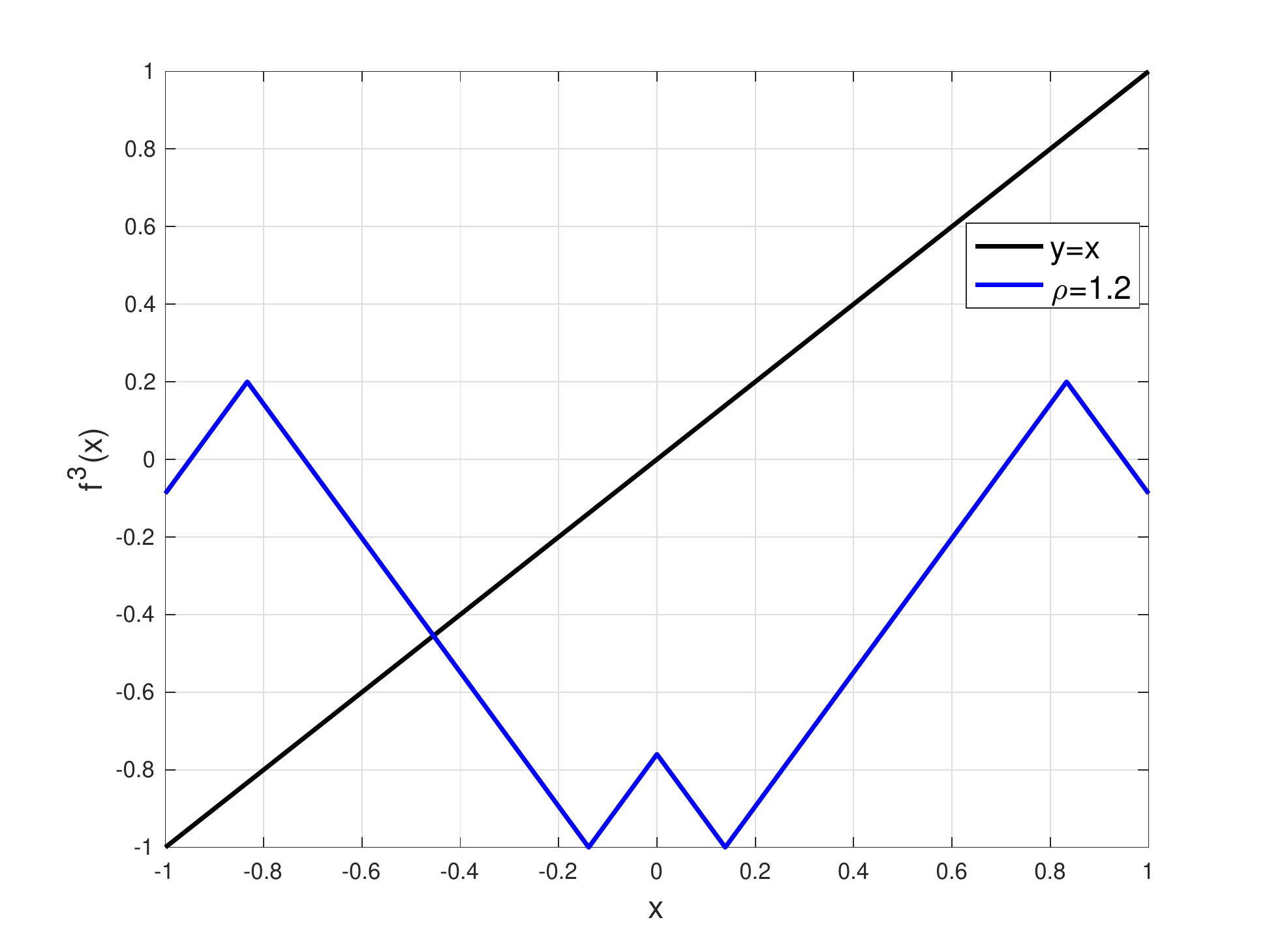}}\\
	\subfigure[Graph of $f^5(x)$ intersected with $y=x$, to identify period 5 points.]{\includegraphics[width = 2.8in]{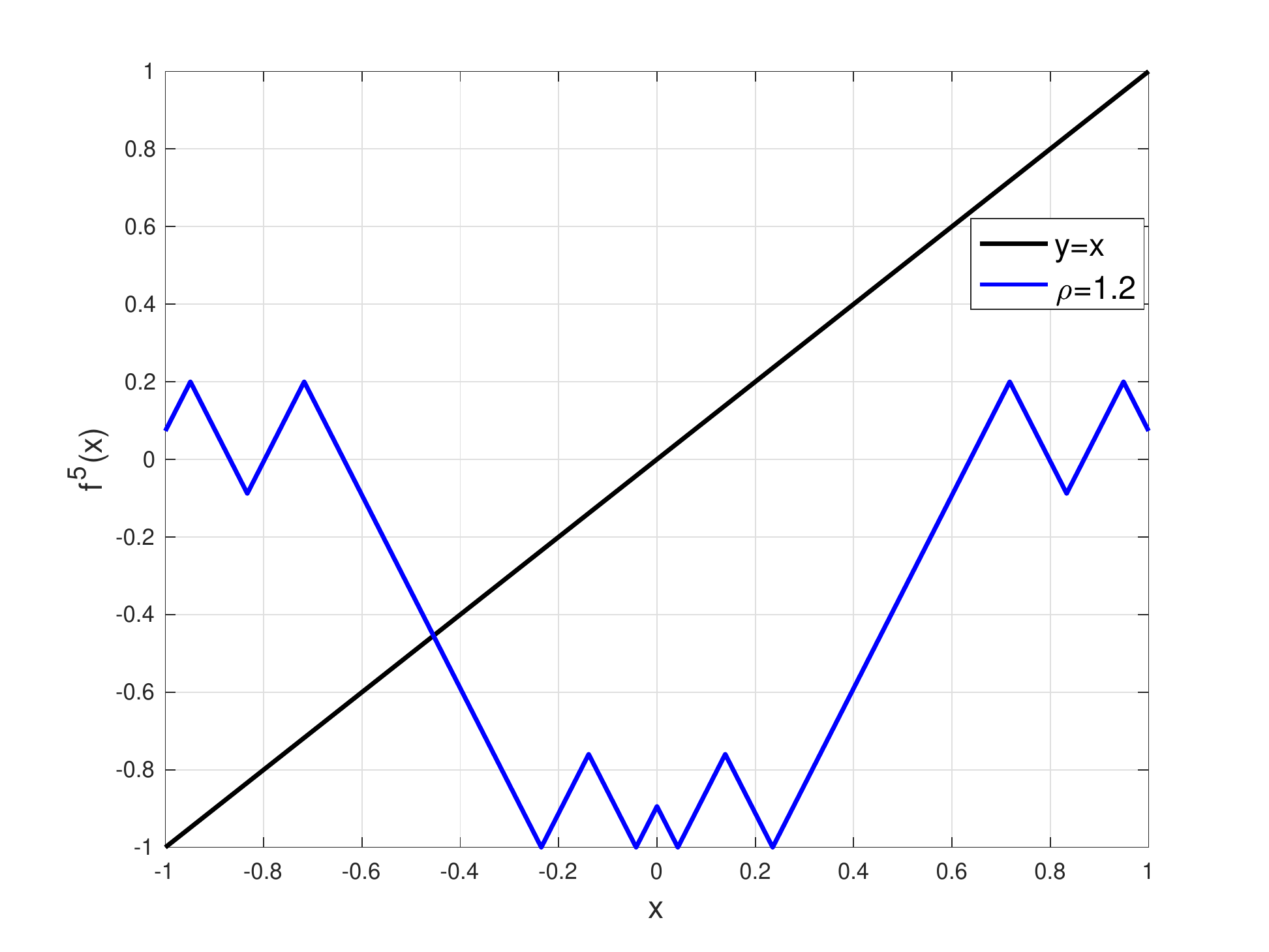}}
	\hspace{2cm}
	\subfigure[Graph of $f^7(x)$ intersected with $y=x$, to identify period 7 points.]{\includegraphics[width = 2.8in]{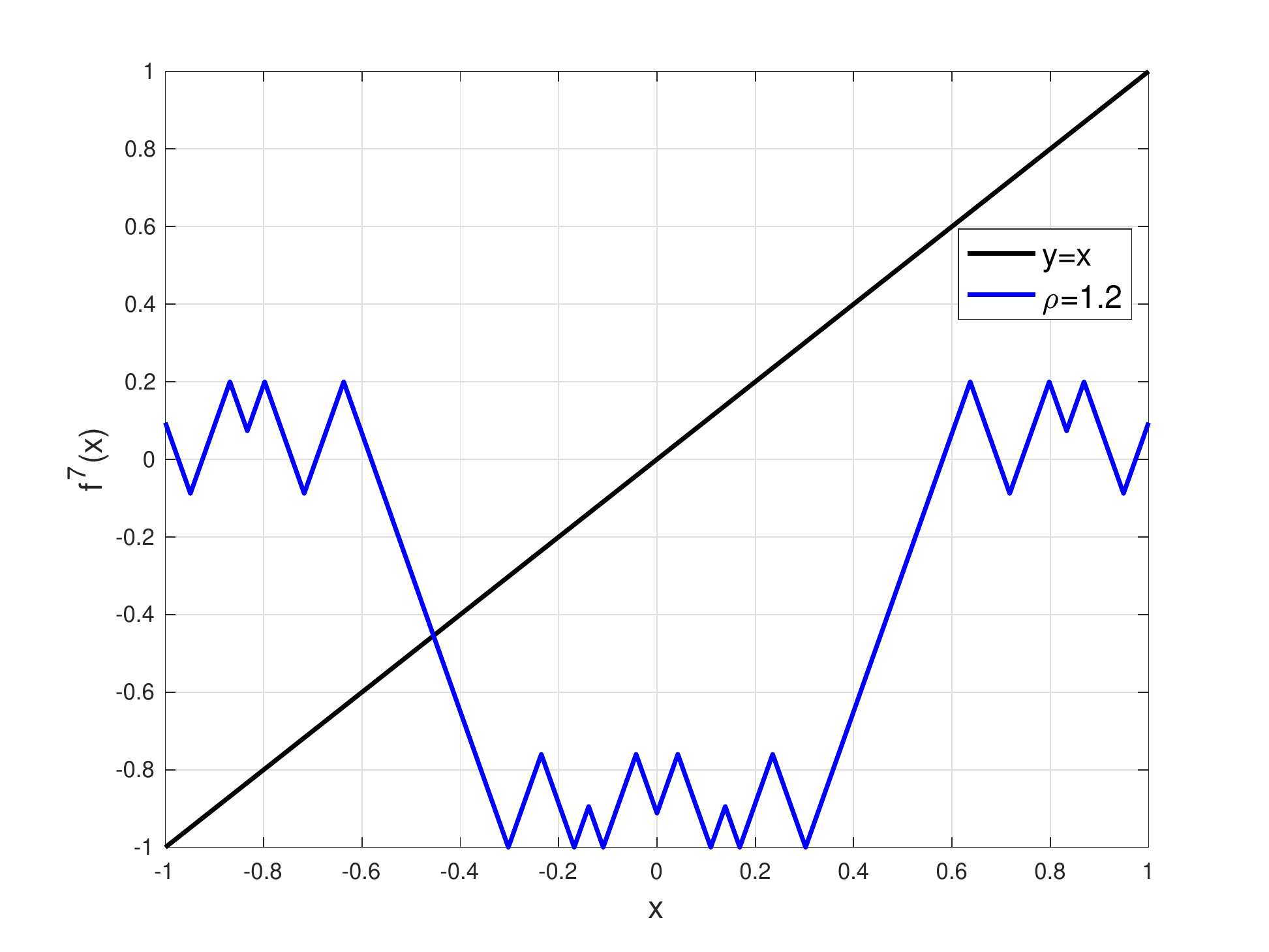}}
	\caption[]
	{\small Here $L=1.2$ that corresponds to the regime where $L<\sqrt{2}$. It follows that this function cannot have any odd period (because then $L\ge \rho\ge \sqrt{2}$). Observe that the oscillations do not grow exponentially fast and they shrink in area, hence no $L^1$ separation is achievable.}
	\label{fig:rho_3}
\end{figure}

\hyperref[fig:rho_2]{Figure~\ref{fig:rho_2}} and \hyperref[fig:rho_3]{Figure~\ref{fig:rho_3}} correspond to cases where the Lipschitz constant of the function does not match $\rho_{p}$.

\begin{itemize}	
\item When $ \sqrt2 \le L\le \phi $, we see from \hyperref[fig:rho_1]{Figure \ref{fig:rho_1}}, how small differences in the values of the slope can lead to the existence of different (prime) periods, which consequently lead to different depth-width trade-offs.
\item When $L > \phi$, we can see from \hyperref[fig:rho_2]{Figure \ref{fig:rho_2}} that $L=2$ and also the growth rate of oscillations is 2. This means that $L=\rho$ and that $L^1$ separation is achievable. Note that period 3 is present in the tent map, so $\rho_{3}=\phi$ for this case.
\item When $L < \sqrt{2}$, we can see from \hyperref[fig:rho_3]{Figure \ref{fig:rho_3}} that the oscillations do not grow exponentially with compositions and that the existing ones are of small magnitude, which means that the $L^1$ error can be made arbitrarily small. Observe here that no odd period is present in the function (as this would imply that $L\ge \rho\ge \sqrt{2}$).
\end{itemize}

\section{Experiments}
Our goal here is to experimentally validate our theoretical results by exploring the interplay between optimization/representation error bounds obtained in theory. For instance, in order to understand how training with random initialization works on compositions of periodic functions, we combine the example from \hyperref[lem:example]{Lemma}~\ref{lem:example} together with \hyperref[lem:match]{Theorem}~\ref{lem:match}; in particular, theory suggests that for a fixed width $u$ and depth $l$, as long as the condition stated in the theorem is satisfied, i.e., $(2u)^l \leq \frac{\rho^t}{8}$, then we have an error bound that is \textit{independent} of the width and depth. We indeed validate this and the $L^1$ error we get from theory almost matches the experimental error (see \hyperref[fig:exp]{Figure}~\ref{fig:exp}).

Rewriting the condition of the theorem, for a fixed width and depth, there is a large $t\geq \frac{(l+3)\ln(2)+l\ln(u)}{\ln(\rho)}$, that always produces constant error. To test that, we create a ``hard task'' that satisfies this above equation for depths $l=1,2,3,4,5$, for constant width $u=20$.  On the other end of the spectrum, we create a relatively ``easy task'' (with fewer compositions) and study how the error varies with depth. We define a regression task and we fix the neurons for each layer to be 20. We vary the depth of the NN (excluding the input and the output layer) from $l=1$ to $l=5$, adding one extra hidden layer at a time. We are using the same parameters to train all networks and we require the training error or the mean squared error to tend to 0 during the training procedure, i.e, we try to overfit the data (here we try to demonstrate a representation result, rather than a statistical/generalization result). Thus, during training we use the same parameters to train all the different models using {\sc{Adam}} with epochs being 1500 to enable overfitting. To record the training error, we verify that the training saturates based on the error of the last epoch.
We will now describe the ``easy'' and the ``hard'' task:
\paragraph {The Regression Task:} We create a regression task that is based on the theory that ties periods with Lipschitz constants. The function we want to fit is the composition of $f(x):=\rho|x|-1$ with itself. For $\rho=1.618$ (golden ratio), the Lipschitz is also $\rho$ and exhibits period 3.
The hardness of the task is characterized by the number of compositions of the function $f$, as the oscillations increase exponentially fast with the number of compositions. Thus for the hard task we use $f^{40}(x)$ and for the easy task we use $f^8(x)$.

\begin{figure*}[!ht]
	\centering
	\begin{subfigure}[]{\label{fig:hard_task}
			\includegraphics[width=0.45\linewidth]{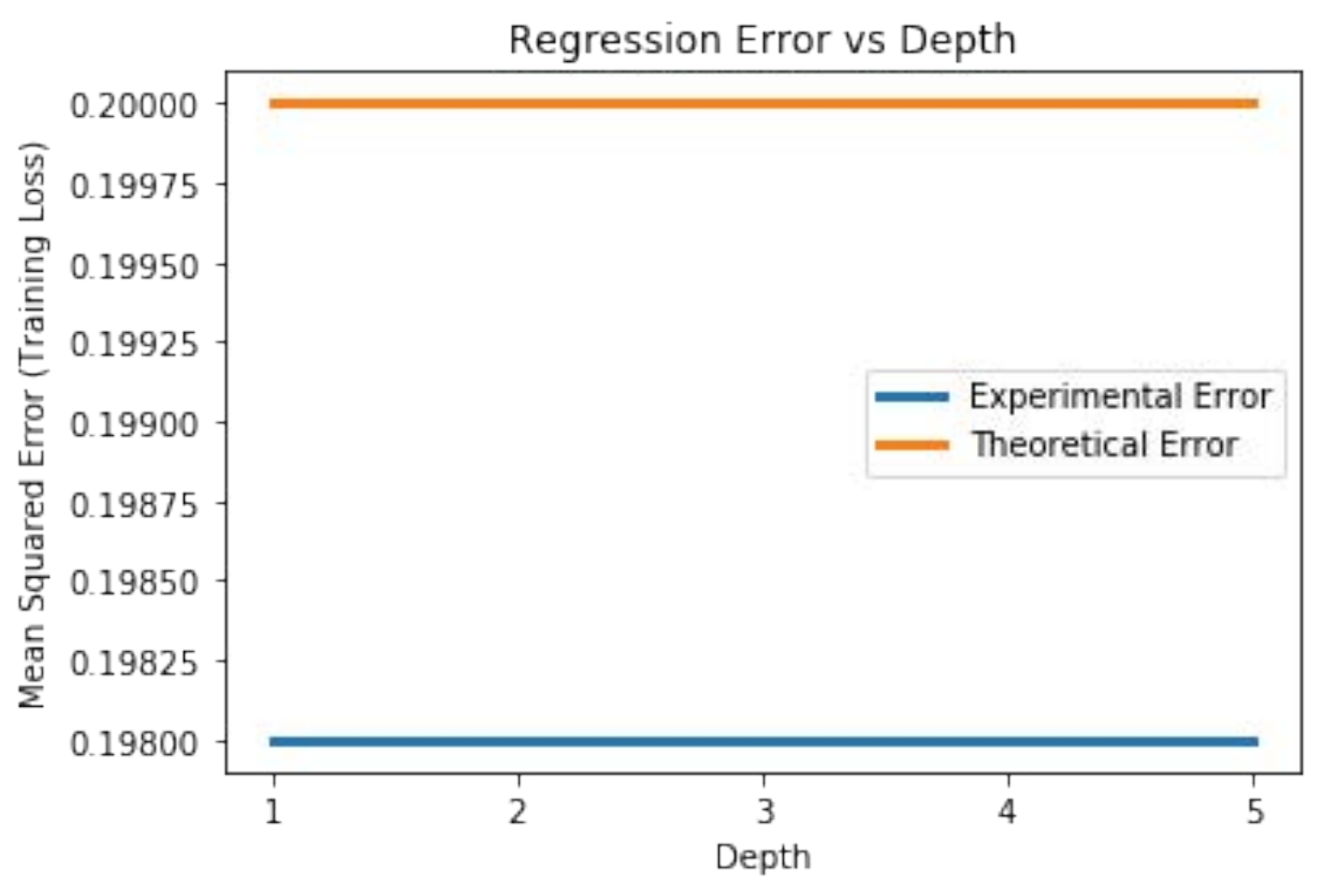}}
	\end{subfigure}%
	\begin{subfigure}[]{\label{fig:easy_task}
			\includegraphics[width=.45\linewidth]{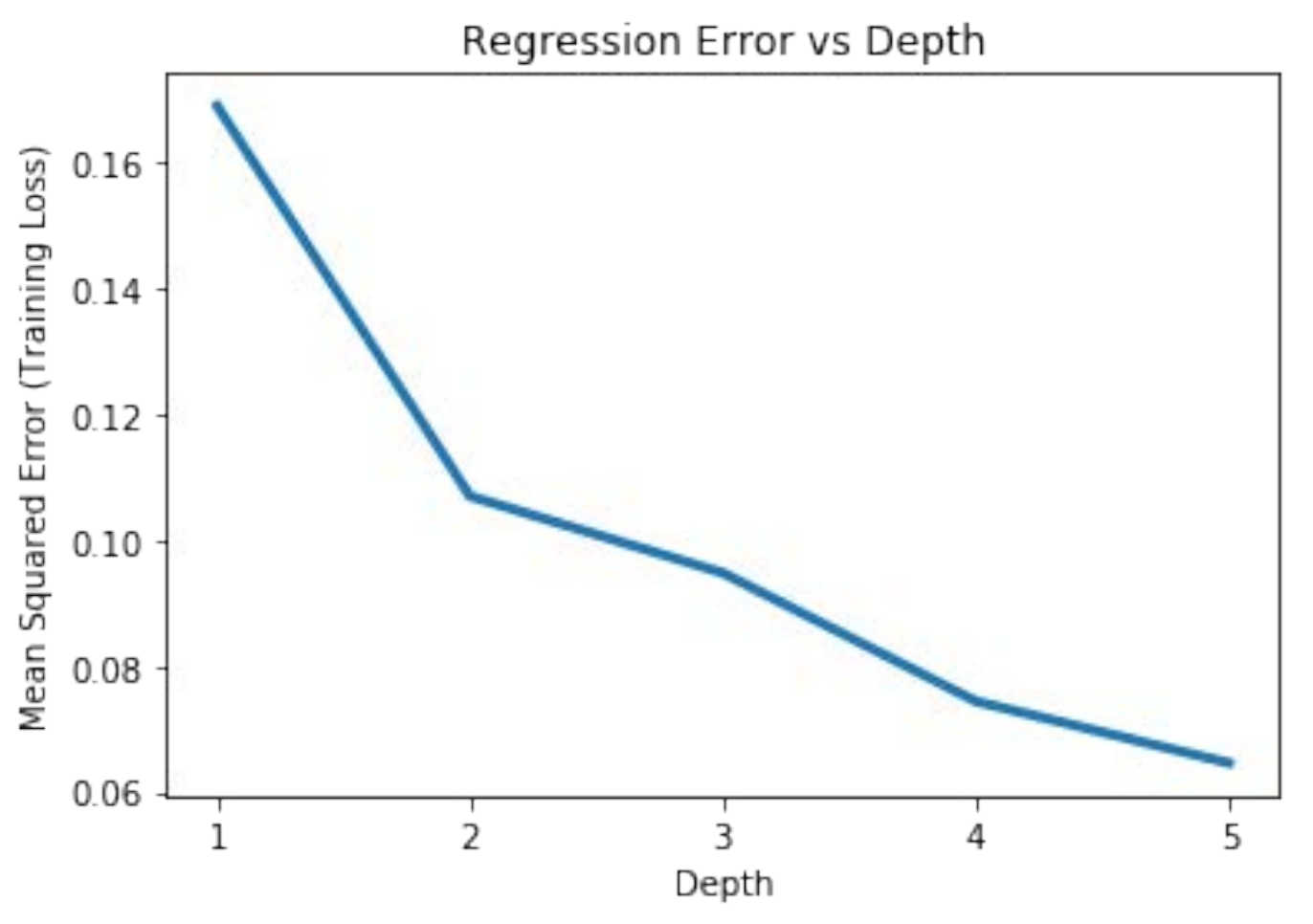}}  
	\end{subfigure}
	\caption{Left: Experimental error closely matches the theoretical bound for the hard task ($f^{40}(x)$) and we get the $L^1$ separation (as long as $u,l$ are related as required by \hyperref[lem:match]{Theorem~\ref{lem:match}}). Right: Generally, the performance of the NN for the regression task improves with depth when the task is easy ($f^{8}(x)$).}
	\label{fig:exp}
\end{figure*}

\section{Discussion}
In conclusion, by combining some ideas from dynamical systems related to periodic orbits and growth rate of oscillations, we presented several results for functions that are expressible with NNs of certain depth, yet are hard-to-represent with shallow, wide NNs. These results generalize and unify previous results on depth separations.

One potential direction for future work, that could further unify the different notions of ``complexity'' considered in previous works, is to explore connections between the notion of \textit{topological} or \textit{metric entropy} (see \cite{german}) and oscillations, periods and VC dimension. The goal here would be to derive general results stating that whenever a function $f$ has large topological entropy, it is harder to approximate its repeated compositions using shallow NNs, as opposed to a function with smaller topological entropy.

\section*{Acknowledgements}

Vaggos Chatziafratis is partially supported by an Onassis Foundation Scholarship.
Sai Ganesh Nagarajan would like to acknowledge SUTD President’s Graduate Fellowship (SUTD-PGF).
Ioannis Panageas would like to acknowledge SRG ISTD 2018 136, NRF for AI Fellowship and NRF2019-
NRF-ANR095. Part of this project happened while the authors were visiting the MIT MIFODS program “Complex Structures” and we would like to thank the organizers for their hospitality.

\bibliography{main}
\bibliographystyle{plainnat}

\end{document}